\documentclass[12pt]{article}
\usepackage[left=2.45cm, right=2.45cm, top=2.45cm, bottom=2.45cm]{geometry}
\usepackage[utf8]{inputenc}
\usepackage{amsmath,amssymb,amsfonts, amsthm}
\usepackage{xcolor}
\usepackage{multirow}
\usepackage{booktabs}
\usepackage{graphicx}
\usepackage{caption}
\usepackage{float}
\usepackage{fancyhdr}
\usepackage{subcaption}
\usepackage[
    backend=biber,
    style=alphabetic,
    ]{biblatex}
\DeclareMathAlphabet{\mathmybb}{U}{bbold}{m}{n}

\makeatletter
\newcommand{\subalign}[1]{%
  \vcenter{%
    \Let@ \restore@math@cr \default@tag
    \baselineskip\fontdimen10 \scriptfont\tw@
    \advance\baselineskip\fontdimen12 \scriptfont\tw@
    \lineskip\thr@@\fontdimen8 \scriptfont\thr@@
    \lineskiplimit\lineskip
    \ialign{\hfil$\m@th\scriptstyle##$&$\m@th\scriptstyle{}##$\hfil\crcr
      #1\crcr
    }%
  }%
}
\makeatother

\newcommand{\norm}[1]{\left\lVert #1 \right\rVert}

\newtheorem{theorem}{Theorem}

\newtheorem{lemma}{Lemma}
\newtheorem{assumptions}{Assumptions}

\theoremstyle{remark}
\newtheorem*{remark}{Remark}

\fancypagestyle{alim}{\fancyhf{}\fancyfoot[L]{\fontsize{9}{11} \selectfont$^{\dagger}$Department of Mathematical Sciences, University of Bath 
\\ \ $^\ddagger$Department of Applied Mathematics and Theoretical Physics, University of Cambridge}}

\title{Closing the ODE-SDE gap in score-based diffusion models through the Fokker--Planck equation}

\author{Teo Deveney$^{\dagger}$, Jan Stanczuk$^{\ddagger}$, Lisa Maria Kreusser$^{\dagger}$,\\     Chris~Budd$^{\dagger}$ and Carola-Bibiane Schönlieb$^{\ddagger}$}

\begin{document}

\date{}
\maketitle
\thispagestyle{alim}

 \begin{abstract}
     Score-based diffusion models have emerged as one of the most promising frameworks for deep generative modelling, due to their state-of-the art performance in many generation tasks while relying on mathematical foundations such as stochastic differential equations (SDEs) and ordinary differential equations (ODEs). Empirically, it has been reported that ODE based samples are inferior to SDE based samples. In this paper we rigorously describe the range of dynamics and approximations that arise when training score-based diffusion models, including the true SDE dynamics, the neural approximations, the various approximate particle dynamics that result, as well as their associated Fokker--Planck equations and the neural network approximations of these Fokker--Planck equations. We systematically analyse the difference between the ODE and SDE dynamics of score-based diffusion models, and link it to an associated Fokker--Planck equation. We derive a theoretical upper bound on the Wasserstein 2-distance between the ODE- and SDE-induced distributions in terms of a Fokker--Planck residual. We also show numerically that conventional score-based diffusion models can exhibit   significant differences between ODE- and SDE-induced distributions which we demonstrate using explicit comparisons. Moreover, we show numerically that reducing the Fokker--Planck residual by adding it as an additional regularisation term leads to closing the gap between ODE- and SDE-induced distributions. Our experiments suggest that this regularisation can improve the distribution generated by the ODE, however that this can come at the cost of degraded SDE sample quality.
 \end{abstract}

\section{Introduction}

Score-based \cite{score_matching} and diffusion-based  \cite{sohldickstein2015diffusion_original} generative models have recently been revived and improved, in \cite{song2020generative_score} and \cite{ho2020denoising}. In \cite{song2021sde}, both frameworks have been unified into a single continuous-time approach based on stochastic differential equations and called score-based diffusion models. 
These approaches have received a lot of attention, achieving state-of-the-art performance in likelihood estimation \cite{song2021sde} and unconditional image generation \cite{dhariwal2021diffusion_beats_gans}. Recently another wave of interest has been sparked by publication of two state-of-the-art text-to-image generation models: Stable Diffusion \cite{rombach2022stable_diffusion} and DALL·E \cite{ramesh2022dalle}.

In addition to achieving impressive performance in both image generation and likelihood estimation, score-based diffusion models do not suffer from training instabilities or mode collapse common in other approaches to deep generative modelling \cite{dhariwal2021diffusion_beats_gans, song2021sde}. Moreover, their time complexity in high-resolutions is much better than that of auto-regressive models \cite{dhariwal2021diffusion_beats_gans}. This makes score-based diffusion  models very attractive avenue for the future of deep generative modelling.

Score-based diffusion models convert data into noise through a diffusion process governed by a stochastic differential equation (SDE). Generating new data points is achieved by sampling noise particles and simulating a reverse-time dynamics of this diffusion process, driven by an equation known as the reverse SDE. The reverse SDE has a closed-form expression which depends solely on the time-dependent gradient field (the so-called score) of the logarithm of the perturbed data distribution.  

In addition to the aforementioned stochastic dynamics, diffusion models can   facilitate deterministic dynamics, which are controlled by a probability flow ordinary differential equation (ODE). This ODE framework offers a deterministic method for sampling from a diffusion model and plays a pivotal role in the computation of likelihoods. Using the Fréchet inception distance (FID) score which is a metric used to assess the quality of images created by generative models, the authors in \cite{song2021sde} report the FID scores on image generation tasks using different sampling methods. They demonstrate that the FID scores of the ODE-based sampler are lower than those of SDE-based sampler, implying that the   ODE-based sampler has inferior performance compared to the stochastic counterpart \cite{song2021sde}. However, no explicit comparisons of ODE and SDE samples is provided. 
This raises questions about the validity of the likelihood computations and the theoretical reasons for the discrepancy between SDE- and ODE-induced distributions. 

In this paper, we will analyse the theoretical underpinnings of score-based diffusion models and show that the discrepancy can be explained through a mean-field perspective on diffusion models.

\textbf{Our contributions}:  We rigorously describe the range of dynamics and approximations that arise when training
score-based diffusion model, including the forward SDE dynamics, the neural approximations, the various approximate particle dynamics, as well as the associated mean-field equations and the neural network approximations of the mean-field equations.
 We systematically analyse the difference between ODE and SDE dynamics of score-based diffusion models and link it to an associated Fokker--Planck equation. We derive a theoretical upper bound on the Wasserstein 2-distance between the ODE- and SDE-induced distributions in terms of a Fokker--Planck residual. We show numerically that conventional score-based diffusion models can exhibit   significant differences between SDE- and ODE-induced distributions which we demonstrate using explicit comparisons. Moreover, we show numerically that training score-based diffusion models with an additional Fokker--Planck regularisation term leads to closing the gap between   SDE and   ODE distributions. In our experiments we show that this can improve the distribution generated by the ODE sampler, though it can have a negative impact on SDE samples.

\subsection{Related work}

The deterministic ODE dynamics for score-based diffusion models were introduced in \cite{song2021sde}, where the authors show that under a perfect score approximation the SDE and ODE distributions coincide and derive a method for computing the likelihoods based on the ODE formulation. In the same work, the authors report that empirically under imperfect score approximation the ODE sampler exhibits inferior performance. This empirical finding highlights the necessity for a more rigorous theoretical investigation into this phenomenon.

In \cite{song2021maximum}, the authors derive an upper-bond on the  Kullback–Leibler divergence between the SDE-induced model distribution and the target data-generating distribution in terms of the score-matching objective that is minimised to train score-based diffusion models. However, they also point out that the same bound does not hold for the ODE-induced distribution.

This issue is further explored in \cite{lu2022maximumODE}, where the authors introduce a new equality which can be used for bounds of the Kullback–Leibler divergence between the ODE-induced distribution and the data-generating distribution. Their findings reveal that the conventional score matching objective, typically employed in score-based diffusion models, fails to adequately control the error in the ODE distribution. To address this, the authors propose a novel training scheme that optimises an upper bound on the Kullback–Leibler divergence based on higher orders of the score-matching error. 

The authors in  \cite{lai2023fpdiffusion} relate the Kullback–Leibler divergence between the ODE-induced distribution and the data-generating distribution to the error in the Fokker--Planck equation associated with the diffusion process. The authors derive the PDE obeyed by the score of the forward diffusion process and call this the score-Fokker--Planck equation. By further bounding the upper bound derived in \cite{lu2022maximumODE}, they demonstrate that the residual of the score-Fokker--Planck equation can control the ODE sample error up to some additive constant, and therefore minimising the score-Fokker--Planck residual reduces an upper bound on the log-likelihood of the probability flow ODE. In principle their results can be used to quantify the discrepancy between ODE- and SDE-induced distributions by relating each to the data generating distribution. However, additional convergence of the score matching objective to zero is required to establish convergence between these distributions. Furthermore, they develop a numerical regularisation scheme for score-based diffusion models that facilitates the minimisation of the score-Fokker--Planck residual.

Our work also considers the Fokker--Planck equation underlying the diffusion dynamics, and therefore shares some similarities with that of \cite{lai2023fpdiffusion}. However, our analysis has been conducted independently using a different theoretical toolbox, and accordingly reveals different insights. First and foremost, our focus is on quantifying the discrepancy between the ODE and SDE distributions, rather than the distance between the ODE  and the data-generating distribution. We postulate that this is a more relevant quantity to examine since the probability flow ODE is derived through a reformulation of the underlying Fokker--Planck equation. We derive bounds on the ODE-SDE discrepancy in terms of a Fokker--Planck residual that do not contain any additive constants, nor do they rely on the value of the score-matching objective. This is significant, as our numerical experiments suggest that in practice that trying to enforce agreement with the underlying Fokker--Planck equation adversely affects the score matching objective, which makes the concurrent minimisation of these terms unrealistic for a fixed network architecture. Secondly, our analysis is done in terms of Wasserstein 2-distance rather than  Kullback–Leibler divergence, and as such is a distinct mathematical approach. We postulate that in this context the Wasserstein distance is a more desirable quantity than  Kullback–Leibler divergence, because the results remain meaningful even for mutually singular distributions. Such scenario could arise if the data distribution is supported on a low dimensional sub-manifold and SDE and ODE distributions approximate different manifolds which do not align perfectly. For a detailed discussion on the issues of the   Kullback–Leibler divergence in relation to   distributions supported on sub-manifolds, we   refer to \cite{arjovsky2017principled} and \cite{arjovsky2017wasserstein}. Thirdly, our analysis is based on a Fokker--Planck equation formulated for the log-density function as opposed to the score function, and we refer to this log-density as a potential. Consequently, in our numerical experiments, we employ the potential parameterisation, rather than the conventional score-parameterisation of the neural network. The choice of the potential parameterisation does not only allow  us to introduce a regularisation term that minimises the residual of the log-density Fokker--Planck equation, but also ensures that the resulting score approximation is a conservative vector field. This is a desirable property, since the ground truth solution is a gradient field. However, as demonstrated in our numerical analysis, the conventional parameterisation often results in score approximations with a non-zero curl. The potential parameterisation and conservation property of the vector fields obtained by score parameterisation has also been  explored in \cite{salimans2021potential}. While the potential parameterisation is the focus of our work, we remark that it is simple to adapt our theory to the score parameterisation setting and attain analogous bounds in terms of a score-Fokker--Planck residual similar to the one considered  in \cite{lai2023fpdiffusion}. A sketch of this reasoning is also provided.

\subsection{Outline}
In Section \ref{sec:diffusionmodels}, we describe the broad range of dynamics and approximations that arise when training a score-based diffusion model. Our main theoretical result on the ODE-SDE gap in score-based diffusion models is proven in Section \ref{sec:mainresult} where we derive an upper bound on the Wasserstein 2-distance between the ODE- and SDE-induced distributions in terms of a Fokker--Planck residual. In Section \ref{sec:numerics} we provide numerical evidence showing explicitly that conventional score-based diffusion models can exhibit  significant differences between SDE- and ODE-induced distributions. Moreover we show here that reducing the Fokker--Planck residual by adding it as an additional regularisation term indeed leads to closing the gap between SDE and ODE distributions.

\section{Score-based diffusion models}\label{sec:diffusionmodels}

\subsection{Assumptions and notation} \label{sec:assumptions}
 We will work in the time domain $t\in [0,T]$ for some $T>0$ and spatial domain $x \in \Omega \subset \mathbb{R}^d$. For two vectors $a,b\in \mathbb{R}^d$ we denote their inner product $a \cdot b:=a^Tb $, with associated norm $\|a\|_2 := (a\cdot a)^{1/2}$. For a function $h: \Omega \to \mathbb{R}^n$ we denote the $L^2$-norm over some domain $\Omega\subset \mathbb R^d$ as $\|h\|_{L^2(\Omega)} := \left(\int_\Omega \|h(x)\|_2^2 dx\right)^{1/2}$.    We denote by $\nabla$ the (spatial) gradient and by $\nabla^2$ the diffusion. For two probability measures $p,q$ on $\Omega$, we denote their Wasserstein 2-distance by $W_2(p,q)$. 
 
 Let $(\mathmybb{\Omega},\mathbb{F},\mathbb{P})$ be a probability space, and let $\mathcal{F}_t \subset \mathbb{F}$ be the natural filtration (the increasing family of sub-$\sigma$-algebras containing information at times $[0,t]$). As is convention, we denote by $W_t\in \mathbb R^d$ a Brownian motion  at time $t$ with values in $\mathbb R^d$ adapted to the filtration $\mathcal{F}_t$. Conversely, let $\bar{\mathcal{F}}_t \subset \mathbb{F}$ denote a reverse filtration (the decreasing family of sub-$\sigma$-algebras containing information at times $[t,T]$). We denote by $\bar{W}_t\in \mathbb R^d$ a Brownian motion at time $t$ with values in $\mathbb R^d$ adapted to $\bar{\mathcal{F}}_t$. Under suitable assumptions, stochastic differential equations (SDEs) driven by $W_t$ are adapted to $\mathcal{F}_t$, and SDEs driven by $\bar{W}_t$ are adapted to $\bar{\mathcal{F}}_t$. Throughout we will refer to the former as \emph{forward} SDEs, and the latter as \emph{reverse} SDEs even though both SDEs will initially be formulated using the forward time variable $t$. When dealing with SDEs and their mean-field limit  we will distinguish between evolution equations running forward and backwards in time by introducing the reverse time variable  $\tau= T-t$ to specify that the corresponding dynamics are in reverse time. We will denote the SDE dynamics parameterised with the forward and reverse time variables $t$ and  $\tau= T-t$  by $x_t$ and  $\bar x_\tau$, respectively, with $\bar x_\tau =\bar x_{T-t}=x_t=x_{T-\tau}$. For any function $h\colon \mathbb R^d \times [0,T]\to \mathbb R$, we introduce $\bar h\colon \mathbb R^d \times [0,T]\to \mathbb R$ by $\bar h(\cdot,\tau)=h(\cdot,T-\tau) $ for all $\tau \in [0,T]$. Further, let probability densities $p_0, \pi$ on $\mathbb R^d$ be given and we denote the associated log-densities by $u_0 = \log p_0$, $u_T = \log \pi$ on $\mathbb R^d$.
Throughout the paper, we make the following regularity assumptions:
\begin{assumptions}\label{ass:regularity}
Let $T>0$ and let $f\in C^{\infty}(\mathbb R^d \times [0,T];\mathbb R^d)$ such that $\|f(x,t)\|_2 \leq K_f(1+\|x\|_2)$ for some $K_f>0$. Assume that $g \in C^{\infty}([0,T]; \mathbb R)$ and there is $0<m<M<\infty$ such that $m\leq g(t)\leq M$ for  all $t\in[0,T]$. We assume that $\Omega \subset \mathbb{R}^d$ is a bounded domain with $\partial\Omega \in C^{\infty}$. For neural approximations we assume smooth activation functions, so that neural potential models $u_\theta:\mathbb R^d \times [0,T]\to \mathbb{R}$ are in $C^{\infty}(\mathbb R^d \times [0,T]; \mathbb R)$ throughout, and neural score models $s_\theta:\mathbb R^d \times [0,T]\to \mathbb{R}^d$ are in $C^{\infty}(\mathbb R^d \times [0,T];\mathbb R^d)$. Moreover we assume that there are $K_u,K_s>0$ such that $\|\nabla u_\theta(x,t)\|_2 \leq K_u(1+\|x\|_2)$ and $\|s_\theta(x,t)\|_2 \leq K_s(1+\|x\|_2)$. Finally, we assume that the second moments of $p_0$ and $\pi$ are finite, and that $\pi \in C^{\infty}(\mathbb R^d; \mathbb R)$.
\end{assumptions}

\subsection{Particle dynamics}

We introduce the \emph{forward SDE}  as
\begin{align}\label{eq:forward}
    dx_t = f (x_t,t)dt + g(t)dW_t,
\end{align}
equipped with some initial distribution $p_0$   for $x_0$. In generative modelling settings, this initial distribution $p_0$  represents the underlying target distribution from which the data was sampled. In \eqref{eq:forward}, $W_t$ denotes the value of a Brownian motion adapted to $\mathcal{F}_t$, and therefore $x_t$ is also adapted to $\mathcal{F}_t$. We denote the associated marginal density of samples from \eqref{eq:forward} at time $t$ by $p(\cdot,t)$ with $p_0=p(\cdot,0)$. Note that \eqref{eq:forward} has a unique $t$-continuous solution by Assumptions \ref{ass:regularity}.
 
In \cite{anderson1982reverse_time_sde}, the author shows that the process in \eqref{eq:forward} can be written as an SDE measurable with respect to the reverse filtration $\bar{\mathcal{F}}_t$. We refer to this SDE as the \emph{reverse SDE} and it is given by
 \begin{align}\label{eq:reverse}
      dx_t =  ( f(x_t,t) - g^2(t) \nabla  \log p(x_t,t))dt + g(t) d\bar{W}_t,
 \end{align}
 where $\bar{W}_t$ is a Brownian motion adapted to $\bar{\mathcal{F}}_t$ at time $t$. Intuitively one can think of $\bar{W}_t$ as the backwards evolution of Brownian motion with known terminal state, and \eqref{eq:reverse} as the backwards evolution of \eqref{eq:forward}. Accordingly, if the terminal distribution for $x_T$ is set to $p(\cdot,T)$, then the trajectories of \eqref{eq:reverse} share the same distribution as \eqref{eq:forward} for any time $t\in [0,T]$.
As shown in \cite{song2021sde}, a reformulation of the Fokker--Planck equations allows us to derive the \emph{probability flow ODE} of the forward SDE \eqref{eq:forward}. This is given by
\begin{align}\label{eq:probflow}
    \frac{dx_t}{dt} =   f(x_t,t) - \frac{1}{2} g^2(t)  \nabla \log p(x_t,t),
\end{align}
equipped with initial distribution $p_0$ for $x_0$ or, equivalently, terminal distribution $p(\cdot,T)$ for $x_T$. The trajectories initialised from $p_0$  evolve forward in time according to \eqref{eq:probflow} and also have marginal distribution  $p(\cdot,t)$ at time $t$. Similarly, the trajectories with terminal condition $x_T$ sampled from $p(\cdot,T)$ have marginal distribution  $p(\cdot,t)$ at time $t$.
Therefore we have that the associated densities to \eqref{eq:forward}, \eqref{eq:reverse} and \eqref{eq:probflow} are all given by $p(\cdot,t)$ at any time $t\in[0,T]$.

\subsection{Neural approximation}

For generative tasks, practitioners  assume $p(\cdot,T)$ to be equal to a given \emph{prior distribution} $\pi$ and simulate equations \eqref{eq:reverse} or \eqref{eq:probflow} to generate samples from $p_0$.  Typically, $\pi$ approximates $p(\cdot,T)$ and is an easy to sample from distribution that contains no information of $p_0$, such as a Gaussian distribution with fixed mean and variance.  However, solving  equations \eqref{eq:reverse} or \eqref{eq:probflow}  requires knowledge of the \emph{(Stein) score function} $\nabla \log p(x_t,t)\in \mathbb R^d$ for any $x_t$, which is not known in general and must be approximated from data. Therefore a neural network $s_\theta(x_t,t)\in \mathbb R^d$ with model parameters $\theta$ is trained to approximate the score function from the data by minimising the weighted score matching objective
\begin{gather}
\label{SM}
\begin{aligned}
    \mathcal{L}_{SM}(\theta, s_\theta,\lambda) := 
     \mathbb{E}_{\subalign{&t \sim U(0,T)\\ &x_t \sim p(x_t, t)}} [\lambda(t) \norm{\nabla{\log{p(x_t, t)}} - s_\theta(x_t,t)}_2^2]
\end{aligned}
\end{gather}
where $\lambda: [0,T] \xrightarrow{} \mathbb{R}_+$ is a positive weighting function.

$\mathcal{L}_{SM}$ in \eqref{SM} cannot be minimised directly since we do not have access to the ground truth score $\nabla{\log{p(x_t, t)}}$. Therefore, in practice, a different objective has to be used \cite{score_matching, vincent2011connection, song2021sde}. In \cite{song2021sde}, the weighted denoising score-matching objective is considered, which is defined as 
\begin{gather}\label{DSM}
\begin{aligned}
    \mathcal{L}_{DSM}(\theta, s_\theta, \lambda) := 
     \mathbb{E}_{\subalign{&t \sim U(0,T)\\ &x_0 \sim p_0(x_0) \\ &x_t \sim p(x_t, t | x_0, 0)}} [\lambda(t) \norm{\nabla{\log{p(x_t, t | x_0, 0)}} - s_\theta(x_t,t)}_2^2].
\end{aligned}
\end{gather}
The difference between \eqref{SM} and \eqref{DSM} is the replacement of the unknown ground truth score $\nabla{\log{p(x_t, t)}}$  by the score of the perturbation kernel $ \nabla \log p(x_t, t | x_0, 0) $ which  can be determined analytically for many choices of forward SDEs. Note that for a fixed function $\lambda$, objective \eqref{DSM} is equal to   objective \eqref{SM} up to an additive constant, which does not depend on the  model parameters  $\theta$. The reader can refer to \cite{vincent2011connection} for the proof. 

The choice of the weighting function $\lambda$ is   important  because it determines the quality of score-matching in different diffusion scales. A principled choice for the weighting function is $\lambda(t) = g(t)^2$. This weighting function is called the likelihood weighting function.

\begin{remark}
  The choice  $\lambda(t) = g^2(t)$  ensures that \eqref{DSM} together with the Kullback–
  Leibler divergence $D_{KL}$ from the true terminal distribution $p(\cdot,T)$ to the given prior distribution $\pi$ yields an 
 upper bound on the Kullback–Leibler divergence from the target distribution $p(\cdot,0)$ to the model distribution $p_{\theta}^{SDE}(\cdot,0)$. Here,  $p_{\theta}^{SDE}(\cdot,0)$ refers to the distribution of samples obtained by simulating a neural approximation of particle dynamics \eqref{eq:reverse}, which we introduce in \eqref{eq:reverseapprox}. More precisely, it holds that
\begin{equation*}
\begin{aligned}
    D_{KL}(p(\cdot,0) \parallel p_{\theta}^{SDE}(\cdot,0)) &\leq  D_{KL}(p(\cdot,T)\parallel \pi(\cdot)) + \frac{T}{2} \mathcal{L}_{SM}(\theta, s_\theta, g^2).
\end{aligned}
\end{equation*}
 Other weighting functions  also yielded very good results \cite{kingmaVDM} for particular choices of forward SDEs. However, there are no theoretical guarantees that alternative weightings would yield good results for arbitrary choices of forward SDEs.

Similarly to  the model distribution $p_{\theta}^{SDE}(\cdot,0)$ of the  neural approximation of the particle dynamics \eqref{eq:reverse}, one can also consider the distribution of samples $p_{\theta}^{ODE}(\cdot,0)$ obtained by simulating the neural approximation of particle dynamics \eqref{eq:probflow} and we will  formally introduce $p_{\theta}^{ODE}(\cdot,0)$ in Section \ref{sec:approxparticles}. A bound of the  Kullback–Leibler divergence from the target distribution $p(\cdot,0)$ to the model distribution $p_{\theta}^{ODE}(\cdot,0)$ 
has been derived in \cite{lu2022maximumODE} and is given by
\begin{equation*} 
D_{KL}(p(\cdot, 0) \parallel p_{\theta}^{ODE}(\cdot, 0)) = D_{KL}(p(\cdot, T) \parallel \pi(\cdot)) + \frac{T}{2}\mathcal{L}_{SM}(\theta, s_\theta, g^2)  + \frac{T}{2}\mathcal{L}_{Diff}(\theta, s_\theta, g^2),
\end{equation*}
where
\begin{align*}
&\mathcal{L}_{Diff}(\theta, s_\theta, \lambda) \\&\qquad:=  \mathbb{E}_{\subalign{&t \sim U(0,T)\\ &x_t \sim p(x_t, t)}}  \left[ \lambda(t) \left(\nabla \log p(x,t) - s_\theta(x,t) \right)^T \left( \nabla \log p_{\theta}^{ODE}(x,t) - s_\theta(x,t) \right) \right] dt.
\end{align*}
Upper bounds for $\mathcal{L}_{Diff}(\theta, s_\theta, g^2)$ have been derived in \cite{lu2022maximumODE,lai2023fpdiffusion}, and training schemes based on minimising these upper bounds have been proposed.
\end{remark}

Most implementations of neural score approximations parameterise the time-dependent score vector field directly with a neural network $s_{\theta}: \Omega \times [0, T] \xrightarrow{} \mathbb{R}^d$ on some bounded domain $\Omega$. Our experiments illustrated in Figure \ref{fig:curl} show that such approximation results in a vector field, which is not conservative and therefore cannot be a gradient field of any function. Since we know a priori that the target vector field $\nabla \log p$ is a gradient field, instead of learning $s_\theta$, we consider a neural network $\phi_\theta \in C^\infty(\Omega\times [0,T];\mathbb R)$ such that $\phi_\theta(x,t)$ approximates the log-density $\log p(x,t)$ for any $(x,t)\in \Omega\times [0,T]$ up to some normalising constant. In other words there exists a (time-dependent) normalising constant $Z_t\in \mathbb R$ such that $$p_\theta(x,t) = \exp(\phi_\theta(x,t) - \log Z_t)$$ is a probability distribution. We write $u_\theta = \log{p_\theta}$ for the induced log-density and we call the function $\phi_\theta(x,t)$ a \emph{potential model}. During training the induced approximate score $\nabla \phi_\theta(x,t) = \nabla u_\theta(x,t) \approx \nabla \log p(x, t)$ is computed  by back-propagation through $\phi_\theta$ with respect to the input $x$. This results in a score approximation that is provably a conservative vector field. Moreover, it enables us to calculate the time derivative of the approximate log-density (up to normalisation) as $\partial_t u_\theta(x,t) \approx \partial_t \log p(x,t)$ by back-propagation through $\phi_\theta$ with respect to $t$. This will prove crucial later, when we  introduce and evaluate a log-Fokker--Planck residual for $u_\theta$ in Section \ref{sec:pinns}. 

\begin{figure}[htb]
    \centering
    \captionsetup[subfigure]{labelformat=empty}
    \begin{subfigure}{.45\textwidth}
    \centering
        \includegraphics[width=\textwidth]{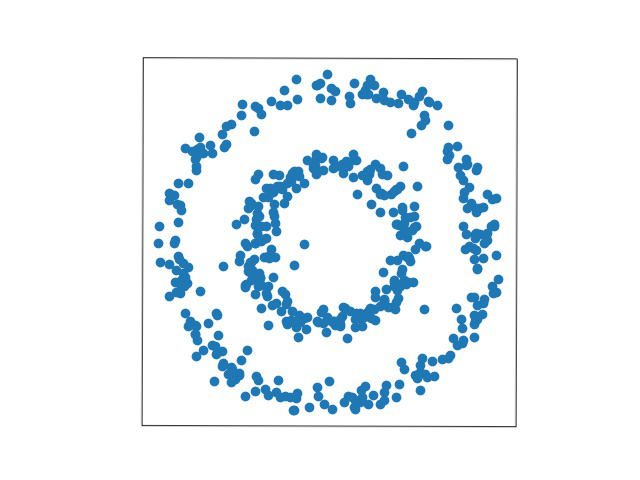}
        \vspace{2mm}
        \caption{(a) Samples from the concentric circles model}
    \end{subfigure} 
    \begin{subfigure}{.45\textwidth}
    \centering
        \includegraphics[width=0.93\textwidth]{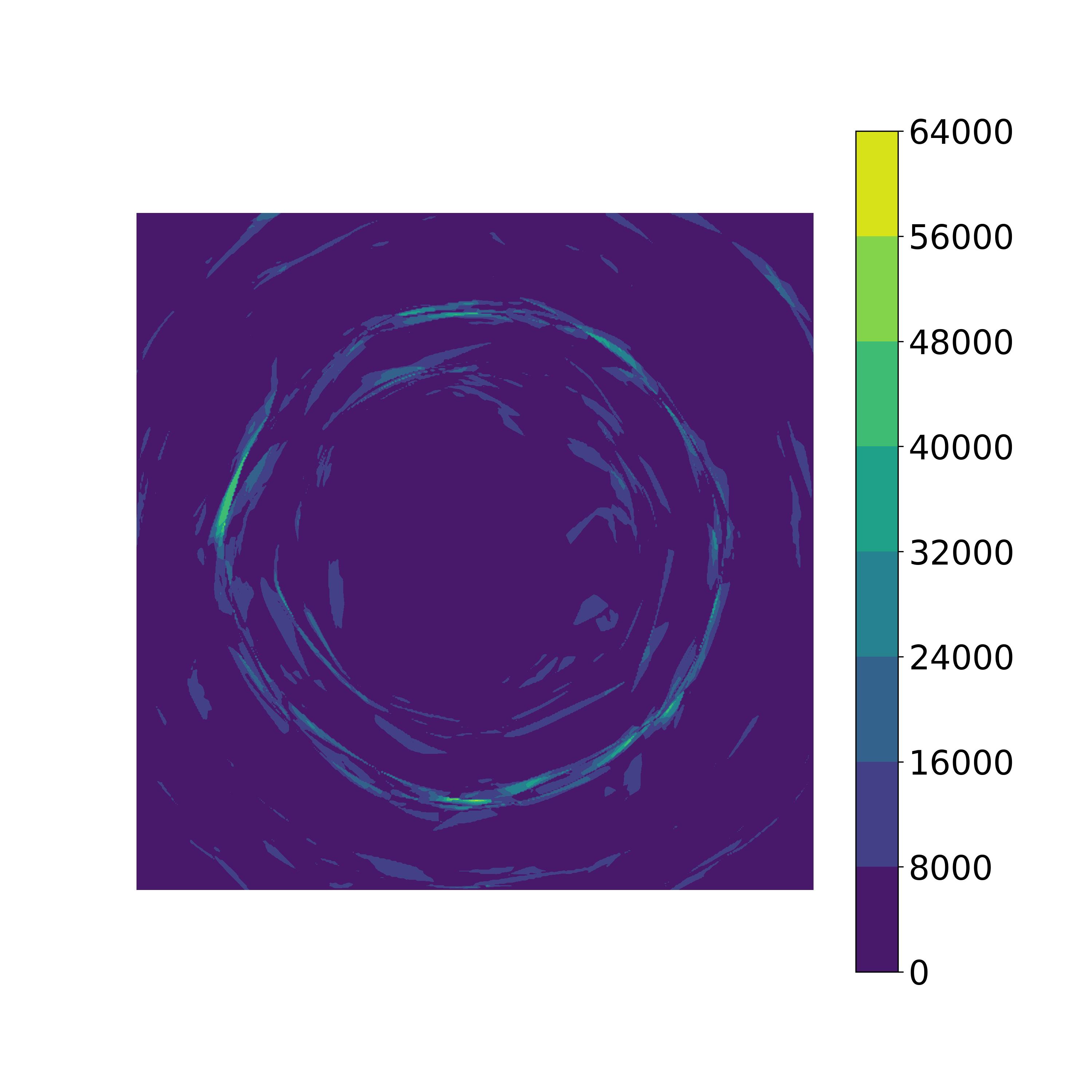}
        \caption{(b) Curl of the concentric circles model}
    \end{subfigure}
    \caption{Samples and Euclidean norm of the curl of a trained (non-potential) score model (i.e. $\|\nabla \times s_\theta(0,x)\|_2$). Clearly, the score model $s_\theta(0,x)$ is not a conservative vector field.}
    \label{fig:curl}
\end{figure}

\subsection{Approximate particle dynamics}\label{sec:approxparticles}

The above neural approximations $u_\theta = \log{p_\theta}\in C^\infty (\Omega \times [0,T];\mathbb R)$  induce  approximate versions of \eqref{eq:reverse} and its deterministic flow \eqref{eq:probflow}. 
For ease of notation, we introduce the  approximated reverse drift, 
\begin{align}\label{eq:approxrevdrift}
    f_\theta^{SDE}(x,t)= f(x,t) - g^2(t) \nabla u_\theta(x,t),    
\end{align} 
obtained by substituting the potential model into the drift of \eqref{eq:reverse}. Note that by the assumed properties of $f,g,u_\theta$ in Assumptions \ref{ass:regularity}, it follows that $f_\theta^{SDE} \in C^{\infty}(\mathbb{R}^d \times [0,T];\mathbb R^d)$ and $\|f_\theta^{SDE}(x,t)\|_2\leq (K_f+M^2K_u)(1+\|x\|_2)$.
Using the approximated reverse drift \eqref{eq:approxrevdrift}, we obtain
the \emph{reverse approximate SDE}
 \begin{align}\label{eq:reverseapprox}
      dx_t =  f^{SDE}_\theta(x_t,t) dt + g(t) d\bar{W}_t,
 \end{align}
 which can be regarded as an approximation of \eqref{eq:reverse}. Here $x_t$ is adapted to the reverse time filtration $\bar{\mathcal{F}}_t$ and by Assumptions \ref{ass:regularity},  \eqref{eq:reverseapprox} has a unique $t$-continuous solution.
We denote the marginal density of $x_t$  satisfying \eqref{eq:reverseapprox} by $p_{\theta}^{SDE}(\cdot,t)$ at time $t$, and equip it with some terminal distribution $\pi$ of $x_T$ at time $T$, i.e., $p_{\theta}^{SDE}(\cdot,T)=\pi$, where $\pi$ is chosen to be a Gaussian approximation of $p(\cdot,T)$. Thus the reverse flow of probability $p_{\theta}^{SDE}$ induced by \eqref{eq:reverseapprox} may be close to $p$ depending on the accuracy of the potential model. 
Applying the result of \cite{anderson1982reverse_time_sde} to write \eqref{eq:reverseapprox} as a process measurable with respect to $\mathcal{F}_t$, we arrive at the \emph{forward approximate SDE}, given by
\begin{align}\label{eq:forwardapprox}
    dx_t =\left[ 
     f^{SDE}_\theta(x_t,t)
    +g^2(t)\nabla \log p^{SDE}_\theta(x_t,t)\right]dt + g(t)dW_t,
\end{align}
where $x_0$ is drawn from $p_{\theta}^{SDE}(\cdot,0)$. The associated probability flow ODE of the approximate SDE (in forward time) is  
\begin{align}\label{eq:probflowapprox}
    \frac{dx_t}{dt} = 
     f^{SDE}_\theta(x_t,t)
    +\frac{1}{2}g^2(t)\nabla \log p^{SDE}_\theta(x_t,t),
\end{align}
where $x_0$ is drawn from $p_{\theta}^{SDE}(\cdot,0)$. Note that the associated densities to \eqref{eq:reverseapprox}, \eqref{eq:forwardapprox} and \eqref{eq:probflowapprox} are all given by $p_{\theta}^{SDE}(\cdot,t)$ for $t\in[0,T]$.
 
Finally, we introduce an approximation of the probability flow ODE \eqref{eq:probflow}  by approximating $\log p$ in \eqref{eq:probflow} by a neural network $u_\theta$. This yields the \emph{approximate probability flow ODE} (in forward time), given by 
\begin{align*}
    \frac{dx_t}{dt} &=f(x_t,t) - \frac{1}{2}g^2(t)\nabla u_\theta(x_t,t)),
 \end{align*}
 or alternatively, 
\begin{align}\label{eq:probflowapprox2}
    \frac{dx_t}{dt} = f^{ODE}_\theta(x_t,t),
\end{align}
 using the approximate forward drift
\begin{align}\label{eq:approxdrift}
    f^{ODE}_\theta(x_t,t)=f(x_t,t) - \frac{1}{2}g^2(t)\nabla u_\theta(x_t,t),
\end{align}
where $f_\theta^{ODE} \in C^{\infty}(\mathbb{R}^d\times [0,T];\mathbb R^d)$. Here  $x_T$ distributed according to $\pi$ the associated density is denoted by $p^{ODE}_\theta(\cdot,t)$ for  $t\in [0,T]$. 
 
In summary, the original formulations \eqref{eq:forward}, \eqref{eq:reverse} and \eqref{eq:probflow}  all have density $p$, the approximations \eqref{eq:reverseapprox}, \eqref{eq:forwardapprox} and \eqref{eq:probflowapprox} obtained by approximating the reverse SDE \eqref{eq:reverse}  all have density $p_{\theta}^{SDE}$ and the approximation of the  probability flow ODE \eqref{eq:probflow} has density $p^{ODE}_\theta$. Moreover, there is a density $p_\theta =\exp(u_\theta)$ implied directly by the neural approximation to log-density. In general, we have that $p \neq p_{\theta} \neq p_{\theta}^{SDE} \neq p_{\theta}^{ODE}$.
 
For the majority of our calculations and numerics it is more convenient to work with logarithms of densities rather than the densities themselves. For each density $p$, we denote the associated log-density by $u$ and refer to $u$ as log-density or potential. That is $u(x,t) = \log p(x,t)$,  $u_{\theta}^{SDE}(x,t) = \log p_{\theta}^{SDE}(x,t)$, $u_{\theta}^{ODE}(x,t) = \log p_{\theta}^{ODE}(x,t)$, and $u_\theta(x,t) = \log p_\theta(x,t)$ for all $(x,t)\in\Omega\times [0,T]$.
 
In addition to considering the dynamics in forward time, we can also introduce the dynamics in reverse time. We denote the reverse time dynamics by  $\bar{x}_\tau$ for $\tau \in [0,T]$  satisfying $\bar x_\tau = x_{T-\tau}$  which implies that  $\bar x_T = x_{0}$ and  $\bar x_0 = x_{T}$ for the initial and terminal conditions. 

As we have a terminal condition $x_T$ for \eqref{eq:reverseapprox} and \eqref{eq:reverseapprox} is stated in forward time, 
the corresponding  parameterisation in reverse time can be useful for obtaining samples satisfying \eqref{eq:reverseapprox}. It is given by   
\begin{gather} \label{eq:inverted_approx_reversesde}
   d\bar{x}_\tau = -\bar f^{SDE}_\theta(\bar{x}_\tau, \tau) dt - \bar g(\tau)dW_\tau,
\end{gather}
where we use the notation from Section \ref{sec:assumptions} and the reverse time variable $\tau = T-t$.
We equip \eqref{eq:inverted_approx_reversesde} with initial condition $\bar x_0$ which is sampled from $\pi$ and
we denote the  distribution of $\bar x_\tau$ at time $\tau$ by $\bar{p}_\theta^{SDE}(\cdot, \tau)$. Note that $\bar{p}_\theta^{SDE}(\cdot, \tau)= p_\theta^{SDE}(\cdot, T-\tau)$ for $\tau \in [0,T]$ with  $\bar{p}_\theta^{SDE}(\cdot, 0) =\pi$. 
This implies that we can sample a particle from the target distribution $p_\theta^{SDE}(\cdot, 0)$ by  sampling $\bar{x}_0$ from $ \pi$ and solving    \eqref{eq:inverted_approx_reversesde} until time $\tau = T$, for instance with the Euler-Maruyama scheme.

Similarly the ODE dynamics  \eqref{eq:probflowapprox2} can be written using the reverse time variable $\tau$ as
\begin{gather}
    \label{eq:inverted_approx_ode}
   \frac{d\bar{x}_\tau}{d\tau} = -\bar{f}_\theta^{ODE}(\bar x_\tau, \tau). 
\end{gather}
To sample from the approximate target distribution $p_\theta^{ODE}(\cdot, 0)=\bar p_\theta^{ODE}(\cdot, T)$, sample an initial condition $\bar x_0$ from $\bar{p}_\theta^{ODE}(\cdot, 0) =\pi$, and simulate \eqref{eq:inverted_approx_ode} forward in time.

\subsection{Mean-field equations}

The mean-field equations describe the evolution of the densities subject to some initial or terminal condition. For the forward SDE \eqref{eq:forward} the density $p$ obeys the {\em forward Fokker--Planck equation} 
\begin{align} \label{eq:fokkerplackforward}
    \frac{\partial p}{\partial t}(x,t) + \nabla \cdot ({f}(x,t)p(x,t)) - \frac{1}{2}g^2(t)\nabla^2p(x,t)=0
\end{align} 
on $\mathbb R^d \times [0,T]$,
equipped with  the initial data $p_0$ on the full space $\mathbb R^d$. For our analysis, we restrict ourselves to a bounded domain  $\Omega \subset \mathbb{R}^d$  with $\partial \Omega \in C^\infty$.  For considering \eqref{eq:fokkerplackforward} on $\Omega$, we equip \eqref{eq:fokkerplackforward} with positive Dirichlet boundary conditions. Let  $p_B\colon \partial \Omega\times [0,T] \to \mathbb R$ denote a positive function which is equal to $p\colon \mathbb R^d \to \mathbb R$ on $\partial \Omega$. Note that we can assume without loss of generality that $p$ is positive on $\partial \Omega$. This yields the forward Fokker--Planck equation \eqref{eq:fokkerplackforward} on the domain $\Omega$ with initial data $p_0$ restricted to $\Omega$ and Dirichlet boundary conditions $p_B$ on $\partial \Omega\times [0,T]$. In addition, we set $u_B=\log p_B$ on $\partial \Omega \times [0,T]$.

The density $p_{\theta}^{SDE}$ of the approximate SDE \eqref{eq:reverseapprox} also satisfies a forward Fokker--Planck equation, which can be derived by writing the Fokker--Planck equation for the forward dynamics \eqref{eq:forwardapprox} of $p_{\theta}^{SDE}$ with appropriate terminal distribution. This gives the {\em approximate  Fokker--Planck equation} (in forward time) 
\begin{align}\label{eq:fokkerplanckapprox}
    \frac{\partial p_{\theta}^{SDE}}{\partial t}(x,t) + \nabla \cdot (f^{SDE}_{\theta
    }(x,t)p_{\theta}^{SDE}(x,t)) +\frac{1}{2}g^2(t)\nabla^2(p_{\theta}^{SDE}(x,t))=0
\end{align}
on $\mathbb R^d \times [0,T]$,
equipped with  the terminal condition  $\pi$ from Assumptions \ref{ass:regularity} on the full space $\mathbb R^d$, i.e., $p_{\theta}^{SDE}(x,T)=\pi(x)$ for all $x\in\mathbb R^d$, where $\pi$ is typically specified as a Gaussian approximation of $p(\cdot,T)$.
For considering \eqref{eq:fokkerplanckapprox} on a bounded domain $\Omega$, we introduce positive Dirichlet boundary conditions.  Let  $p_B^{SDE}\colon \partial \Omega\times [0,T] \to \mathbb R$ denote a positive function which is equal to  $p_{\theta}^{SDE}\colon \mathbb R^d \to \mathbb R$ on $\partial \Omega$. We obtain the approximate  Fokker--Planck equation \eqref{eq:fokkerplanckapprox} on the domain $\Omega$ with initial data $p_0$ restricted to $\Omega$ and Dirichlet boundary conditions $p_B^{SDE}$ on $\partial \Omega \times [0,T]$. We set $u_B^{SDE}=\log p_B^{SDE}$ on $\partial \Omega \times [0,T]$.

Note that for $p_{\theta}^{SDE}$, we  always assume a fixed terminal condition at time $T$ when considering \eqref{eq:fokkerplanckapprox} in $t$ (or an initial condition when considering the evolution in reverse time $\tau$) as $p_{\theta}^{SDE}$ describes the flow of probability backwards from a Gaussian approximation $\pi$ of $p(x,T)$ to some approximation of $p_0$.
Notice that \eqref{eq:fokkerplackforward} and \eqref{eq:fokkerplanckapprox} are of a similar form, apart from the different signs of the diffusion terms.

In addition to considering Fokker--Planck equations for the densities, one can also introduce log-Fokker--Planck equations for the potential.
For the density $p$ satisfying the forward Fokker--Planck equation \eqref{eq:fokkerplackforward} for the forward SDE \eqref{eq:forward} and the associated  potential $u=\log p$ we introduce the \emph{forward log-Fokker--Planck equation} (in forward time) as
\begin{align}\label{eq:logfokkerplanck}
\begin{split}
  \frac{\partial u}{\partial t}(x,t) + \nabla \cdot f(x,t)  & + \nabla u (x,t) \cdot f(x,t)  -\frac{1}{2}g^2(t)\|\nabla u (x,t)\|_2^2-\frac{1}{2}g^2(t)\nabla^2 u (x,t) = 0
  \end{split}
\end{align}
on $\mathbb R^d \times [0,T]$.
On the domain $\Omega$, we equip \eqref{eq:fokkerplackforward}  with initial data $u_0$ restricted to $\Omega$ and boundary conditions $u_B$ on $\partial \Omega\times [0,T]$.

For $p_{\theta}^{SDE}$ solving  \eqref{eq:fokkerplanckapprox} in forward time,  the log-density $u_{\theta}^{SDE} = \log p_{\theta}^{SDE}$ satisfies the \emph{approximate log-Fokker--Planck equation} (in forward time) given by 
\begin{align}\label{eq:logfokkerplanckapprox}
\begin{split}
  \frac{\partial u_{\theta}^{SDE}}{\partial t}(x,t) + \nabla \cdot f_\theta^{SDE}(x,t)  & + \nabla u_{\theta}^{SDE}(x,t) \cdot f_\theta^{SDE}(x,t) \\&\quad +\frac{1}{2}g^2(t)\|\nabla u_{\theta}^{SDE}(x,t)\|_2^2+\frac{1}{2}g^2(t)\nabla^2 u_{\theta}^{SDE}(x,t) = 0
  \end{split}
\end{align}
on $\mathbb R^d \times [0,T]$.
On the domain $\Omega$, we equip \eqref{eq:logfokkerplanckapprox}  with terminal data $u_T$ restricted to $\Omega$ and boundary conditions $u_B^{SDE}$ on $\partial \Omega\times [0,T]$.

Similarly to the the mean-field equations for the SDE dynamics, we can also consider mean-field equations for the ODE dynamics. The deterministic ODE dynamics  \eqref{eq:probflowapprox2}  can be viewed as a continuous normalising flow and can be leveraged to compute data likelihood \cite{song2021sde}. By applying the change of variables formula to the associated continuity equation we obtain  
\begin{gather*}
    \frac{\partial u_\theta^{ODE}}{\partial t}(x_t,t) = -\nabla \cdot f_\theta^{ODE}(x_t,t)
\end{gather*}
for the desired quantity $ u_{\theta}^{ODE}(x_t, t) = \log p_{\theta}^{ODE}(x_t, t)$. We equip $p^{ODE}_{\theta}$ with the terminal condition $\pi$, implying that the unknown term $\log p^{ODE}_{\theta}(x_T, T)$ is given  by the prior log-likelihood $\log \pi(x_T)$.

\subsection{Physics-informed neural networks for mean-field equations}\label{sec:pinns}

Physics informed neural networks (PINN) are deep learning models that approximate the solution to a PDE with some given boundary and initial   conditions by substituting a neural network into these equations and minimising the differential and boundary operator residuals in some norm. In the case of the  approximate Fokker--Planck equation \eqref{eq:fokkerplanckapprox} and the approximate log-Fokker--Planck equation \eqref{eq:logfokkerplanckapprox}, training a PINN would amount to reducing the associated residual.

In the following, we restrict ourselves to a bounded domain  $\Omega \subset \mathbb{R}^d$  with $\partial \Omega \in C^\infty$. 
For  the approximate Fokker--Planck equation \eqref{eq:fokkerplanckapprox}, we consider a neural network with network parameters $\theta$ which is trained to determine a density $p_\theta$ so that $p_\theta$ satisfies positive Dirichlet boundary conditions $p_B^{SDE}$,  terminal condition $p_\theta(\cdot,T)=\pi$ and minimises some appropriate residual.
Note that for an approximate solution $p_\theta$ of \eqref{eq:fokkerplanckapprox}, we obtain that
\begin{align*}
    \begin{split}
    &\frac{\partial p_\theta}{\partial t}(x,t) + \nabla \cdot (f^{SDE}_{\theta}(x,t)p_\theta(x,t)) + \frac{1}{2}g^2(t)\nabla^2(p_\theta(x,t)) \\
    &= \frac{\partial p_\theta}{\partial t}(x,t) + \nabla \cdot (f(x,t)p_\theta(x,t) - g^2(t)\nabla \log p_\theta(x,t) p_\theta(x,t)) + \frac{1}{2}g^2(t)\nabla^2(p_\theta(x,t)) \\
    &=\frac{\partial p_\theta}{\partial t}(x,t) + \nabla \cdot (f(x,t)p_\theta(x,t)) - g^2(t)\nabla^2 p_\theta(x,t) + \frac{1}{2}g^2(t)\nabla^2(p_\theta(x,t)) \\
    &=\frac{\partial p_\theta}{\partial t}(x,t) + \nabla \cdot (f(x,t)p_\theta(x,t)) - \frac{1}{2}g^2(t)\nabla^2(p_\theta(x,t))
    \end{split}
\end{align*}
on $\Omega \times [0,T]$.
This demonstrates that  the residual for the forward  Fokker--Planck equation \eqref{eq:fokkerplackforward} is equivalent to the residual for the approximate Fokker--Planck equation \eqref{eq:fokkerplanckapprox}, and so reducing the residual for the forward Fokker--Planck equation \eqref{eq:fokkerplackforward}  of the forward SDE \eqref{eq:forward} is equivalent to reducing the residual of the approximate Fokker--Planck equation \eqref{eq:fokkerplanckapprox} of the approximate reverse SDE \eqref{eq:reverseapprox}. Hence, we define the residual of the Fokker--Planck equation for the approximate reverse SDE \eqref{eq:reverseapprox} for any $t\in [0,T)$ as 
\begin{align} \label{eq:fpresid}
    R(\theta,p_\theta,t) = V(T-t)^{-1}\int_t^T\left\|\frac{\partial p_\theta}{\partial s}(\cdot,s) + \nabla \cdot (f(\cdot,s)p_\theta(\cdot,s)) - \frac{1}{2}g^2(s)\nabla^2(p_\theta(\cdot,s))\right\|_{L^2(\Omega)}^2 ds, 
\end{align}
where $V(r):=r\operatorname{Vol}(\Omega)$ is the volume of $[T-r,T]\times \Omega$. We refer to $R$ as the \emph{Fokker--Planck residual}.

For deriving the residual corresponding to the approximate log-Fokker--Planck equation \eqref{eq:logfokkerplanckapprox},  
we consider a neural network with network parameters $\theta$ which is trained to determine  $u_\theta$ approximating the solution to \eqref{eq:logfokkerplanckapprox}.
Note that $u_\theta$ satisfies
\begin{align}\label{eq:restransform}
    \begin{split}
    &\frac{\partial u_\theta}{\partial t}(x,t) + \nabla \cdot f^{SDE}_{\theta}(x,t) + \nabla u_\theta(x,t) \cdot f^{SDE}_{\theta}(x,t) + \frac{1}{2}g^2(t)\|\nabla u_\theta(x,t)\|_2^2 + \frac{1}{2}g^2(t) \nabla^2 u_\theta(x,t) \\
    &= \frac{\partial u_\theta}{\partial t}(x,t) + \nabla \cdot (f(x,t)-g^2(t)\nabla u_\theta(x,t)) + \nabla u_\theta(x,t) \cdot (f(x,t)-g^2(t)\nabla u_\theta(x,t)) \\& \qquad  + \frac{1}{2}g^2(t)\|\nabla u_\theta(x,t)\|_2^2 + \frac{1}{2}g^2(t) \nabla^2 u_\theta(x,t) \\
    &= \frac{\partial u_\theta}{\partial t}(x,t) + \nabla \cdot f(x,t) + \nabla u_\theta(x,t) \cdot f(x,t) - \frac{1}{2}g^2(t)\|\nabla u_\theta(x,t)\|_2^2 - \frac{1}{2}g^2(t) \nabla^2 u_\theta(x,t)
    \end{split}
\end{align}
on $\Omega \times [0,T]$.
This implies that residual of the approximate log-Fokker--Planck equation \eqref{eq:logfokkerplanckapprox} for the approximate reverse SDE \eqref{eq:reverseapprox} is equal to the residual of the forward log-Fokker--Planck equation \eqref{eq:logfokkerplanck} for the forward SDE \eqref{eq:forward}.
Hence, we define the \emph{log-Fokker--Planck residual} corresponding to \eqref{eq:logfokkerplanckapprox} as
\begin{align} 
\begin{split} \label{eq:lfpresid}
   \tilde{R}(\theta,u_\theta,t) &=  V(T-t)^{-1}\int_t^T\left\|\frac{\partial u_\theta}{\partial s}(\cdot,s) + \nabla \cdot f(\cdot,s) + \nabla u_\theta(\cdot,s) \cdot f(\cdot,s) \right.\\
    &\qquad\qquad\quad \left.- \frac{1}{2}g^2(s)\|\nabla u_\theta(\cdot,s)\|_2^2 - \frac{1}{2}g^2(s) \nabla^2 u_\theta(x,s)\right\|_{L^2(\Omega)}^2 ds
\end{split}
\end{align}
 for $t\in [0,T)$. These residuals quantify how well our approximations $p_\theta,u_\theta$ agree with the true solutions $p_\theta^{SDE},u_\theta^{SDE}$ to the approximate Fokker--Planck equations. Next we show how the values attained by these residuals defines an upper bound on the ODE-SDE discrepency.

\section{Theoretical results on the ODE-SDE gap}\label{sec:mainresult}

In this section, we investigate the gap between the ODE- and SDE-induced distributions in terms of Fokker--Planck equations. More precisely, we derive bounds related to the approximate log-Fokker--Planck equation \eqref{eq:logfokkerplanckapprox} in Section \ref{sec:result_aFP}, and in Section \ref{sec:result_potential_aFP} we outline how this theory applies to the associated potential model. In Section \ref{sec:result_score_FP} we provide a sketch of how to derive analogous bounds in terms of the approximate score-Fokker--Planck equation, thus addressing the common score parameterisation.

\subsection{The ODE-SDE gap for the  approximate Fokker--Planck equation}\label{sec:result_aFP}

We present some theoretical results showing that, at a fixed time $t\in [0,T]$, $p^{SDE}_\theta(\cdot,t)$ satisfying the approximate Fokker--Planck equation \eqref{eq:fokkerplanckapprox} converges to the density $p^{ODE}_\theta(\cdot,t)$ of the approximate probability flow ODE \eqref{eq:probflowapprox2} with  respect to the Wasserstein 2-distance $W_2$   as the log-Fokker--Planck residual  $\tilde{R}(\theta,u_\theta,t)$ in \eqref{eq:lfpresid} goes to zero. 
More specifically we prove:
\begin{theorem}\label{prop:mainresult}
Let $t\in[0,T)$ and $\delta >0$ be given, let   $\Omega \subset \mathbb{R}^d$ be a bounded domain   with $\partial \Omega \in C^\infty$, and assume that Assumptions \ref{ass:regularity} hold. Assume that the neural network  $u_\theta \in C^\infty (\Omega \times [0,T])$ is determined such that $\tilde{R}(\theta, u_\theta,t)$ in \eqref{eq:lfpresid} satisfies $\tilde{R}(\theta, u_\theta,t)<\delta$, and   the terminal condition $u_T=\log \pi$ restricted to $\Omega$ and Dirichlet boundary conditions $u_B^{SDE}$ on $\partial \Omega\times [0,T]$. Assume that $p^{SDE}_\theta$ satisfies \eqref{eq:fokkerplanckapprox} on   $\Omega$ with terminal condition  $\pi$ restricted to $\Omega$ and Dirichlet boundary conditions $p_B^{SDE}$ on $\partial \Omega \times [0,T]$, with $u_\theta^{SDE}= \log p_\theta^{SDE}$. Further, let $p^{ODE}_\theta$ be the probability density associated with \eqref{eq:probflowapprox2} with terminal condition $\pi$ restricted to $\Omega$. 
Then, $W_2(p^{ODE}_\theta(\cdot,t), p^{SDE}_\theta(\cdot,t))< C\delta$ for some constant $C>0$ independent of $\delta$. 
\end{theorem}

\begin{proof}
We proceed in two steps. Firstly, we show that there exists a constant  $\tilde C>0$ independent of $\delta$ such that $\|u_\theta(\cdot,\tau)- u_{\theta}^{SDE}(\cdot,\tau)\|_{L^2(\Omega)}< \tilde C\delta$  which allows us to show that   $\int_0^\tau\|\nabla u_\theta(\cdot, s)-\nabla u_{\theta}^{SDE}(\cdot, s)\|^2_{L^2(\Omega)}d s <C\delta$ for the  reverse time variable  $\tau = T-t \in (0,T]$ where the constant  $C>0$ is independent of $\delta$.  
Finally, we prove that $$W_2(p_\theta^{ODE}(\cdot,t), p_\theta^{SDE}(\cdot,t))< C\delta$$ for some constant $C>0$ independent of $\delta$.\\
\textbf{Step I:} To prove the first part, assume that $u^{SDE}_\theta$ solves \eqref{eq:logfokkerplanckapprox} on $\Omega\times [0,T]$, that is
\begin{align} \label{eq:fpexactorig}
    \begin{split}
    \frac{\partial u_{\theta}^{SDE}}{\partial t}(x,t)& + \nabla \cdot f_\theta^{SDE}(x,t) + \nabla u_{\theta}^{SDE}(x,t) \cdot f_\theta^{SDE}(x,t)\\
     &+ \frac{1}{2}g^2(t)\|\nabla u_{\theta}^{SDE}(x,t)\|_2^2 + \frac{1}{2}g^2(t) \nabla^2 u_{\theta}^{SDE}(x,t) = 0 
    \end{split}
\end{align}
for $(x,t)\in \Omega\times [0,T]$,
equipped   with terminal data $u_T$ restricted to $\Omega$ and boundary conditions $u_B^{SDE}$ on $\partial \Omega\times [0,T]$.
Further, assume that $u_\theta$ satisfies  \eqref{eq:logfokkerplanckapprox} on  $\Omega\times [0,T]$ with some residual $q$, i.e.
\begin{align} \label{eq:fpapproxorig}
    \begin{split}
    \frac{\partial u_\theta}{\partial t}(x,t)& + \nabla \cdot f_\theta^{SDE}(x,t) + \nabla u_\theta(x,t) \cdot f_\theta^{SDE}(x,t)\\
    +& \frac{1}{2}g^2(t)\|\nabla u_\theta(x,t)\|_2^2 + \frac{1}{2}g^2(t) \nabla^2 u_\theta(x,t) =q(x,t) 
    \end{split}
\end{align}
for $(x,t)\in \Omega\times [0,T]$,
equipped   with terminal data $u_T$ restricted to $\Omega$ and boundary conditions $u_B^{SDE}$ on $\partial \Omega\times [0,T]$.

As \eqref{eq:fpexactorig} and \eqref{eq:fpapproxorig} are equipped with terminal conditions, it is more natural to work with the corresponding reverse time equations.  Writing the reverse time variable as $\tau = T-t \in [0,T]$ we have 
\begin{align} \label{eq:fpexact}
    \begin{split}
    \frac{\partial \bar u_{\theta}^{SDE}}{\partial \tau}(x,\tau)& - \nabla \cdot \bar f_\theta^{SDE}(x,\tau) - \nabla \bar u_{\theta}^{SDE}(x,\tau) \cdot \bar f_\theta^{SDE}(x,\tau) \\
    &- \frac{1}{2}\bar g^2(\tau)\|\nabla \bar u_{\theta}^{SDE}(x,\tau)\|_2^2 - \frac{1}{2}\bar g^2(\tau) \nabla^2 \bar u_{\theta}^{SDE}(x,\tau) = 0.
    \end{split}
\end{align}
and
\begin{align}\label{eq:fpapprox}
    \begin{split}
    \frac{\partial \bar u_\theta}{\partial \tau}(x,\tau)& - \nabla \cdot \bar f_\theta^{SDE}(x,\tau) - \nabla \bar u_\theta(x,\tau) \cdot \bar f_\theta^{SDE}(x,\tau)\\ 
    & - \frac{1}{2}\bar g^2(\tau)\|\nabla \bar u_\theta(x,\tau)\|_2^2 - \frac{1}{2}\bar g^2(\tau) \nabla^2 \bar u_\theta(x,\tau) = \bar q(x,\tau).
    \end{split}
\end{align}
Note that the log-Fokker--Planck residual \eqref{eq:lfpresid} is related to $\bar q(\cdot,\tau)$ for $\tau =T-t$ by 
\begin{align} \label{eq:q_res}
\begin{split}
  \tilde{R}(\theta,u_\theta,t)   = \tilde{R}(\theta,u_\theta,T-\tau)  
        =V(\tau)^{-1}\int_{T-\tau}^T\| q(\cdot, s)\|^2_{L^2(\Omega)}ds
   = V(\tau)^{-1}\int_{0}^\tau\|\bar q(\cdot , s)\|^2_{L^2(\Omega)}d s
    \end{split}
\end{align}
which follows from \eqref{eq:restransform} and \eqref{eq:fpapproxorig}.
Subtracting \eqref{eq:fpexact} for $\bar u_\theta^{SDE}$ from \eqref{eq:fpapprox} for $\bar u_\theta$ yields
\begin{align*}
    \begin{split}
   & \frac{\partial (\bar u_\theta(x,\tau) - \bar u_\theta^{SDE}(x,\tau))}{\partial \tau}(x,\tau) - \nabla(\bar u_\theta(x,\tau)-\bar  u_\theta^{SDE}(x,\tau)) \cdot \bar f^{SDE}_\theta(x,\tau)\\ 
    &\qquad - \frac{1}{2} \bar g^2(\tau)(\|\nabla \bar u_\theta(x,\tau)\|^2_2 - \|\nabla \bar u_{\theta}^{SDE}(x,\tau)\|_2^2)  -\frac{1}{2}\bar g^2(\tau)\nabla^2(\bar u_\theta-\bar u_\theta^{SDE})(x,\tau)= \bar q(x,\tau).
    \end{split}
\end{align*}
We define the error $e_{\bar u}=\bar u_\theta - \bar u_\theta^{SDE}$ on $\Omega\times [0,T]$. Note that the boundary and terminal conditions of $\bar u_{\theta}^{SDE}$ and $\bar u_\theta$ imply that $e_{\bar u}(x,0)=0$ for all $x\in \Omega$ and that $e_{\bar u}$ has homogeneous Dirichlet boundary conditions. This allows us to write a PDE for the error $e_{\bar u}$ given by
\begin{align*}
  \frac{\partial e_{\bar u}}{\partial \tau}(x,\tau) - \nabla e_{\bar u}(x,\tau) \cdot \bar f^{SDE}_\theta(x,\tau)  - \frac{1}{2}\bar g^2(\tau)(\|\nabla \bar u_\theta(x,\tau)\|^2_2 &- \|\nabla \bar u_{\theta}^{SDE}(x,\tau)\|_2^2) \\&- \frac{1}{2}\bar g^2(\tau)\nabla^2 e_{\bar u} (x,\tau)= \bar q(x,\tau).
\end{align*}
Using  $\frac{1}{2}\frac{\partial e^2_{\bar u}}{\partial \tau} = e_{\bar u}\frac{\partial e_{\bar u}}{\partial \tau}$ we obtain an equation for the squared error $e^2_{\bar u}$ which reads
\begin{align*}
\begin{split}
    \frac{1}{2}\frac{\partial e^2_{\bar u}}{\partial \tau}(x,\tau) = e_{\bar u}(x,\tau)\nabla e_{\bar u}(x,\tau)& \cdot \bar f^{SDE}_\theta(x,\tau) + \frac{1}{2}\bar g^2(\tau)e_{\bar u}(x,\tau)(\|\nabla \bar u_\theta(x,\tau)\|^2_2 - \|\nabla \bar u_{\theta}^{SDE}(x,\tau)\|_2^2)  \\
    &+ \frac{1}{2}\bar g^2(\tau)e_{\bar u}(x,\tau) \nabla^2(e_{\bar u}(x,\tau)) + e_{\bar u}(x,\tau)\bar q(x,\tau).
\end{split}
\end{align*}
We integrate to get an equation for the $L^2$-error given by
\begin{align*}
    \begin{split}
    \frac{1}{2}\frac{\partial}{\partial \tau}\|e_{\bar u}(\cdot,\tau)\|^2_{L^2(\Omega)} 
    &= \int_\Omega e_{\bar u}(x,\tau)\nabla e_{\bar u}(x,\tau)\cdot \bar f_\theta^{SDE}(x,\tau)dx  \\
        & \qquad + \frac{1}{2}\bar g^2(\tau)\int_\Omega e_{\bar u}(x,\tau)(\|\nabla \bar u_\theta(x,\tau)\|^2_2 - \|\nabla \bar u_{\theta}^{SDE}(x,\tau)\|_2^2) dx \\ & \qquad + \frac{1}{2}\bar g^2(\tau)\int_\Omega e_{\bar u}(x,\tau) \nabla^2(e_{\bar u}(x,\tau))dx + \int_\Omega e_{\bar u}(x,\tau)\bar q(x,\tau)dx.
                        \end{split}
\end{align*}
Note that 
\begin{align*}
&\frac{1}{2} \bar g^2(\tau)\int_\Omega e_{\bar u}(x,\tau)(\|\nabla \bar u_\theta(x,\tau)\|^2_2 - \|\nabla \bar u_{\theta}^{SDE}(x,\tau)\|_2^2)dx\\
& = \frac{1}{2}\bar g^2(\tau)\int_\Omega e_{\bar u}(x,\tau)(\nabla \bar u_\theta(x,\tau)+\nabla \bar u_{\theta}^{SDE}(x,\tau))\cdot \nabla e_{\bar u} (x,\tau) dx\\
        & =\frac{1}{4}\bar g^2(\tau)\int_\Omega (\nabla \bar u_\theta(x,\tau)+\nabla \bar u_{\theta}^{SDE}(x,\tau))\cdot \nabla e^2_{\bar u} (x,\tau) dx \\        
        & = - \frac{1}{4}\bar g^2(\tau)\int_\Omega (\nabla^2 \bar u_\theta(x,\tau)+\nabla^2 \bar u_{\theta}^{SDE}(x,\tau))  e^2_{\bar u} (x,\tau) dx      
\end{align*}
by integration by parts together with the homogeneous boundary conditions of $e_{\bar u}$. Using  $\frac{1}{2}\frac{\partial e^2_u}{\partial \tau} = e_{\bar u}\frac{\partial e_{\bar u}}{\partial \tau}$ and applying integration by parts again yields
\begin{align*}
 \begin{split}
    \frac{1}{2}\frac{\partial}{\partial \tau}\|e_{\bar u}(\cdot,\tau)\|^2_{L^2(\Omega)}       
    &= -\frac{1}{2}\int_\Omega e^2_{\bar u}(x,\tau) \nabla \cdot \bar {f}^{SDE}_\theta(x,\tau)dx \\
        & \qquad - \frac{1}{4}\bar g^2(\tau)\int_\Omega (\nabla^2 \bar u_\theta(x,\tau)+\nabla^2 \bar u_{\theta}^{SDE}(x,\tau))  e^2_{\bar u} (x,\tau)dx\\& \qquad - \frac{1}{2}\bar g^2(\tau)\int_\Omega \|\nabla(e_{\bar u}(x,\tau))\|_2^2dx + \int_\Omega e_{\bar u}(x,\tau)\bar q(x,\tau)dx.     
    \end{split}
\end{align*}
From Lemma \ref{lem:diffusionbound} and $\bar u_\theta \in C^\infty(\Omega \times [0,T])$, it follows that there exists $L\in \mathbb R$ such that $\nabla^2 \bar u_\theta+\nabla^2 \bar u_{\theta}^{SDE} \geq L$. Then,
\begin{align}\label{eq:estimate}
\begin{split}
    \frac{1}{2}\frac{\partial}{\partial \tau}\|e_{\bar u}(\cdot,\tau)\|^2_{L^2(\Omega)} 
    &\leq \frac{1}{2}\|\nabla \cdot \bar{f}^{SDE}_\theta(\cdot,\tau)\|_{L^{\infty}(\Omega)}\|e_{\bar u}(\cdot,\tau)\|^2_{L^2(\Omega)}-\frac{1}{4}\bar g^2(\tau)L\|e_{\bar u}(\cdot,\tau)\|^2_{L^2(\Omega)}  \\ 
        &\qquad  -\frac{1}{2}\bar g^2(\tau)\|\nabla(e_{\bar u}(\cdot,\tau))\|_{L^2(\Omega)}^2 + \frac{1}{2\epsilon}\|e_{\bar u}(\cdot,\tau)\|^2_{L^2(\Omega)} + \frac{\epsilon}{2}\|\bar q(\cdot,\tau)\|^2_{L^2(\Omega)}, 
    \end{split}
\end{align}
which follows from applying  Young's inequality and holds for any $\epsilon >0$. The  Poincar\'e inequality $K\|e_{\bar u}(\cdot,\tau)\|_{L^2(\Omega)}\leq   \|\nabla e_{\bar u}(\cdot,\tau)\|_{L^2(\Omega)}$  for some constant $K>0$ yields
\begin{align*}
\begin{split}
\frac{1}{2}\frac{\partial}{\partial \tau}\|e_{\bar u}(\cdot,\tau)\|^2_{L^2(\Omega)} 
   & \leq \frac{1}{2}\|\nabla \cdot \bar {f}^{SDE}_\theta(\cdot,\tau)\|_{L^{\infty}(\Omega)}\|e_{\bar u}(\cdot,\tau)\|^2_{L^2(\Omega)}  -\frac{1}{4}\bar g^2(\tau)L\|e_{\bar u}(\cdot,\tau)\|^2_{L^2(\Omega)}   \\ 
        & \qquad - \frac{1}{2}\bar g^2(\tau) K\|e_{\bar u}(\cdot,\tau)\|_{L^2(\Omega)}^2+\frac{1}{2\epsilon}\|e_{\bar u}(\cdot,\tau)\|^2_{L^2(\Omega)} + \frac{\epsilon}{2}\|\bar q(\cdot,\tau)\|^2_{L^2(\Omega)}.
\end{split}
\end{align*} 
Integration in time yields
\begin{align*}
\begin{split}
        \|e_{\bar u}(\cdot,\tau)\|^2_{L^2(\Omega)}
        \leq& \|e_{\bar u}(\cdot,0)\|^2_{L^2(\Omega)} + \epsilon\int_0^\tau\|\bar q(\cdot,s)\|^2_{L^2(\Omega)}ds \\
         &+\int_0^\tau\left(\|\nabla \cdot \bar{f}^{SDE}_\theta(\cdot,s)\|_{L^{\infty}(\Omega)} - \frac{1}{2}\bar g^2(s)(2K + L) + \frac{1}{\epsilon}\right)\|e_{\bar u}(\cdot,s)\|^2_{L^2(\Omega)}ds,
\end{split}
\end{align*}
where the term $\|e_{\bar u}(\cdot,0)\|^2_{L^2(\Omega)} $ vanishes due to the zero  initial condition of $e_{\bar u}$.
For $\epsilon$ sufficiently small, the  bracketed term in the last integral is positive. Therefore, we can apply Gronwall's inequality to get 
\begin{align} \label{eq:growall}
\begin{split}
  &  \|e_{\bar u}(\cdot,\tau)\|^2_{L^2(\Omega)}\\& \leq   \epsilon\int_0^\tau\|\bar q(\cdot,s)\|^2_{L^2(\Omega)}ds\ 
    \text{exp}\left(\int_0^\tau \left(\|\nabla \cdot \bar{f}^{SDE}_\theta(\cdot,s)\|_{L^{\infty}(\Omega)} - \frac{1}{2}\bar g^2(s)(2K+L) + \frac{1}{\epsilon}\right)ds\right).
    \end{split}
\end{align}
Using \eqref{eq:q_res} and the fact that $g$ is bounded, this proves that $\|e_{\bar u}(\cdot,\tau)\|^2_{L^2(\Omega)} < C\delta$ for some constant $C>0$ independent of $\delta$ and $\tau$, but dependent on $T$ and $\epsilon$. 

Next, we deduce that  $\int_0^\tau\|\nabla e_{\bar u}(\cdot, s)\|^2_{L^2(\Omega)}d s< C\delta$  for some $C>0$ independent of $\delta$. Note that \eqref{eq:estimate} implies
\begin{multline}  \label{eq:gradbound}
   \underset{s\in [0,\tau]}{\text{min}}\{\bar g^2(s)\}\int_0^\tau\|\nabla e_{\bar u}(\cdot,s)\|^2_{L^2(\Omega)}ds \leq 
   \epsilon\int_0^\tau\|\bar q(\cdot,s)\|^2_{L^2(\Omega)}ds - \|e_{\bar u}(\cdot,\tau)\|^2_{L^2(\Omega)} \\
     + \int_0^\tau\left(\|\nabla \cdot \bar{f}^{SDE}_\theta(\cdot,s)\|_{L^{\infty}(\Omega)} -\frac{1}{2}\bar g^2(s)L + \frac{1}{\epsilon}\right)\|e_{\bar u}(\cdot,s)\|^2_{L^2(\Omega)}ds.
\end{multline}
It follows from \eqref{eq:growall} that for any $s \in [0,\tau]$  we have $\|e_{\bar u}(\cdot,s)\|^2_{L^2(\Omega)} < C\delta$. 
Since $g$ has a positive lower bound by Assumptions \ref{ass:regularity} 
and $\int_{0}^\tau\|\bar q(\cdot,s)\|^2_{L^2(\Omega)}ds < \delta$ by \eqref{eq:q_res}, this yields $\int_0^\tau\|\nabla e_{\bar u}(\cdot,s)\|^2_{L^2(\Omega)}ds < C\delta$  for some $C>0$ independent of $\delta$. \\
\textbf{Step II:} Finally, we prove that $W_2(p^{ODE}_\theta(\cdot,t), p^{SDE}_\theta(\cdot,t))<C\delta$.  As $e_{\bar u}=\bar u_\theta -\bar u_{\theta}^{SDE}$ in reverse time $\tau$, we have $e_{ u}= u_\theta - u_{\theta}^{SDE}$ in forward time $t=T-\tau$ with  $\nabla e_{ u}=\nabla u_\theta - \nabla u_{\theta}^{SDE}$.
Starting from the probability flow ODE of the approximate SDE \eqref{eq:probflowapprox} and the approximate  drifts $    f^{SDE}_\theta$ in \eqref{eq:approxrevdrift} and     $f^{ODE}_\theta$ in \eqref{eq:approxdrift}, we obtain in forward time 
\begin{align*}
   \frac{dx_t}{dt} &= f^{SDE}_\theta(x_t,t) +\frac{1}{2} g^2(t)\nabla u_{\theta}^{SDE}(x_t,t)\\
    &= f(x_t,t) -  g^2(t)\nabla u_\theta(x_t,t) +\frac{1}{2} g^2(t)(\nabla u_\theta(x_t,t) - \nabla e_u(x_t,t)) \\
    &= f(x_t,t) - \frac{1}{2} g^2(t)\nabla u_\theta(x_t,t) - \frac{1}{2} g^2(t)\nabla e_u(x_t,t) \\
    &= f^{ODE}_\theta(x_t,t) - \frac{1}{2} g^2(t)\nabla e_u(x_t,t).
\end{align*}
This implies that  trajectories traced out by particles obeying \eqref{eq:probflowapprox} with density $p^{SDE}_\theta$ are close to particles obeying \eqref{eq:probflowapprox2} with density $p^{ODE}_\theta$ provided the error $\nabla e_u§$
is small. The densities $p_{\theta}^{SDE}$ and $p^{ODE}_\theta$ induced by the probability flow ODEs \eqref{eq:probflowapprox}, \eqref{eq:probflowapprox2} are associated with  $\bar p_{\theta}^{SDE}$ and $\bar p^{ODE}_\theta$ in reverse time with the corresponding  probability flows in reverse time given by
\begin{align}\label{eq:probflowSDErev}
   \frac{d \bar x_\tau}{d\tau} 
    &= -\bar f^{ODE}_\theta(\bar x_\tau,\tau) + \frac{1}{2}\bar g^2(\tau)\nabla e_{\bar u}(\bar x_\tau,\tau)   
\end{align}
and 
\begin{align}\label{eq:probflowODErev}
   \frac{d\bar x_{\tau}}{d\tau} 
    &= -\bar f^{ODE}_\theta(\bar x_\tau,\tau),   
\end{align}
respectively. 
Applying a change of variables  to continuity equations associated with \eqref{eq:probflowODErev} and \eqref{eq:probflowSDErev} yields  
\begin{align*}
    \frac{d\bar u^{ODE}_\theta(\bar x_\tau,\tau)}{d\tau} &= \nabla \cdot \bar f^{ODE}_\theta(\bar x_\tau,\tau)
    \end{align*}
    and
    \begin{align*}
    \frac{d\bar u^{SDE}_\theta(\bar x_\tau,\tau)}{d\tau} 
    &=  \nabla \cdot \bar f^{ODE}_\theta(\bar x_\tau,\tau) - \frac{1}{2}\bar g^2(\tau)   \nabla^2 e_{\bar u}(\bar x_\tau,\tau)
\end{align*}
for $\bar u^{ODE}_\theta= \log \bar p^{ODE}_\theta$
and $\bar u^{SDE}_\theta=\log  \bar p^{SDE}_\theta$,
respectively.
Using $\frac{d\log \bar p}{d\tau} = \frac{1}{\bar p}\frac{d\bar p}{d\tau}$ for density $\bar p$, we obtain
\begin{align*}
    \frac{d \bar p^{ODE}_\theta(\bar x_\tau,\tau)}{d\tau} &= \bar p^{ODE}_\theta(\bar x_\tau,\tau)\,\nabla\cdot \bar f^{ODE}_\theta (\bar x_\tau,\tau)
    \end{align*}
    and
    \begin{align*}
    \frac{d\bar p^{SDE}_\theta(\bar x_\tau,\tau)}{d\tau} &= \bar p^{SDE}_\theta(\bar x_\tau,\tau)\,\nabla \cdot\bar f^{ODE}_\theta(\bar x_\tau,\tau) - \frac{1}{2}\bar g^2(\tau) \bar p^{SDE}_\theta(\bar x_\tau,\tau) \nabla^2 e_{\bar u}(\bar x_\tau,\tau),  
\end{align*}
respectively.
Then, the error  $\bar e =\bar p^{SDE}_\theta -\bar p_\theta^{ODE} $ between $\bar p^{SDE}_\theta$ and $\bar p_\theta^{ODE}$ satisfies
\begin{align*}
    \frac{d\bar e (\bar x_\tau,\tau)}{d\tau} = \bar e(\bar x_\tau,\tau) \nabla \cdot \bar f^{ODE}_\theta (\bar x_\tau,\tau) - \frac{1}{2}\bar g^2(\tau)\bar p^{SDE}_\theta(\bar x_\tau,\tau) \nabla^2 e_{\bar u}(\bar x_\tau,\tau), 
\end{align*}
equipped with the initial condition $\bar e(\cdot,0)=0$ as $\bar p_{\theta}^{SDE}$ and $\bar p^{ODE}_\theta$ satisfy the same  initial conditions. 

Since $\overline \Omega$ is a compact set, 
convergence in the Wasserstein-2 distance $W_2$ between two measures is equivalent to the convergence in the dual Sobolev space $H^{-1}$ of $H^1=W^{1,2}$.
For $\phi \in H^1(\Omega)$ with $\|\phi\|_{H^1(\Omega)} \leq 1$, we have
\begin{align*}
   & \frac{d}{d\tau}\int_\Omega \phi(x)  \bar e(x,\tau) dx\\ &= \int_\Omega  \phi(x)  \frac{d \bar e(x,\tau)}{d\tau}dx \\
    &= \int_\Omega  \phi(x) \bar e(x,\tau) \nabla\cdot \bar f^{ODE}_\theta(x,\tau) dx - \frac{1}{2}\bar g^2(\tau)\int_\Omega  \phi(x) \bar p^{SDE}_\theta(x,\tau) \nabla^2 e_{\bar u}(x,\tau) dx.
\end{align*}
Integrating with respect to $\tau$ yields
\begin{align*}
\int_\Omega  \phi(x) \bar e(x,\tau) dx &= \int_\Omega  \phi(x)\bar  e(x,0) dx + \int_0^\tau  \int_\Omega  \phi(x)  \nabla \cdot \bar f^{ODE}_\theta(x,s)\bar e(x,s)dx ds\\&\quad - \frac{1}{2} \int_0^\tau  \bar g^2(s)\int_\Omega  \phi(x) \bar p^{SDE}_\theta(x,s) \nabla^2 e_{\bar u}(x,s) dx ds.   
\end{align*}
As  $\bar e(\cdot,0) = 0$, this yields
\begin{align*}
   & \left| \int_\Omega  \phi(x) \bar e(x,\tau) dx \right| \\&\leq \int_0^\tau \left(\left|\int_\Omega  \phi(x) \bar e(x,s)  \nabla \cdot \bar f^{ODE}_\theta(x,s)   dx \right| + \frac{1}{2}\bar g^2(s) \left|\int_\Omega \bar p^{SDE}_\theta(x,s)\phi(x)  \nabla^2 e_{\bar u}(x,s) dx \right| \right)ds \\
    & \leq    \int_0^\tau  \left\|  \nabla\cdot \bar f^{ODE}_\theta(\cdot,s) \right\|_{L^{\infty}(\Omega)} \left|\int_\Omega  \phi(x) \bar e(x,s) dx \right| ds\\&\quad + \frac{1}{2}\int_0^\tau \bar g^2(s) \left\|\bar p^{SDE}_\theta(\cdot,s)\right\|_{L^\infty(\Omega)} \left|\int_\Omega  \nabla \phi(x)  \cdot \nabla e_{\bar u}(x,s) dx \right|ds \\
    &\leq \int_0^\tau \left\|  \nabla \cdot\bar f^{ODE}_\theta(\cdot,s) \right\|_{L^\infty(\Omega)} \left|\int_\Omega  \phi(x)  \bar e(x,s) dx \right|ds \\&\quad + \frac{1}{2} \int_0^\tau \bar g^2(s) \left\|\bar p^{SDE}_\theta(\cdot,s)\right\|_{L^\infty(\Omega)} \|\nabla\phi\|_{L^2(\Omega)}  \|\nabla e_{\bar u}(\cdot,s)\|_{L^2(\Omega)}ds.
\end{align*}
Applying Gronwall's inequality to this, we obtain the estimate
\begin{align*}
    &  \left| \int_\Omega  \phi(x) \bar e(x,\tau) dx \right| \\&\leq \frac{1}{2}\int_0^\tau  \bar g^2(s) \left\|\bar p^{SDE}_\theta(\cdot,s)\right\|_{L^\infty(\Omega)} \|\nabla\phi\|_{L^2(\Omega)}  \|\nabla e_{\bar u}(\cdot,s)\|_{L^2(\Omega)}ds \\
     &\qquad \exp\left(\int_0^\tau \left\|  \nabla\cdot \bar f^{ODE}_\theta(\cdot,s) \right\|_{L^\infty(\Omega)} ds\right)\\
     &\leq C  \int_0^\tau   \left\|\bar p^{SDE}_\theta(\cdot,s)\right\|_{L^\infty(\Omega)}  \|\nabla e_{\bar u}(\cdot,s)\|_{L^2(\Omega)}ds
\end{align*}
for some constant $C>0$ depending on $\bar f_\theta^{ODE}$, $\bar g$ and $T$. The uniform boundedness of $\bar p^{SDE}_\theta$ implies that $ \left| \int_\Omega  \phi(x) \bar e(x,\tau) dx \right| < C\delta$. Note that $ \left| \int_\Omega  \phi(x) \bar e(x,\tau) dx \right|$  
 goes to zero as $\int_0^\tau\|\nabla e_{\bar u} (\cdot,s) \|_{L^2(\Omega)}ds $ goes to zero. 
We have shown therefore that $\bar e = \bar p^{SDE}_\theta -\bar  p^{ODE}_\theta$ vanishes in $H^{-1}$ as $\int_0^\tau\|\nabla e_u (\cdot,s)\|_{L^2}ds$ goes to zero and therefore it follows that the 2-Wasserstein distance between $p^{SDE}_\theta$ and $p^{ODE}_\theta$ goes to zero.

\end{proof}

\subsection{The ODE-SDE gap for the potential model associated with the approximate Fokker--Planck equation}\label{sec:result_potential_aFP}

A key benefit of score-based models is that scores are agnostic to multiplicative scaling of the underlying density, implying known normalising constants are not required for their implementation. So far we have defined the log-Fokker--Planck equation \eqref{eq:logfokkerplanckapprox} as the process governing the logarithm of the density $p_\theta^{SDE}$, and thus have implicitly assumed that $u_\theta$ approximates the logarithm of a normalised density. In practice, controlling the integral of a neural approximation is nontrivial, and so we apply an unnormalised network $\phi_\theta$ as the potential model. If we assume $u_\theta$ as applied above is an approximate potential such that $\exp(u_\theta)$ is a normalised density, then we can relate this to $\phi_\theta$ by introducing a (potentially time-varying)  normalising constant $Z_t$  for $\exp(\phi_\theta(\cdot,t))$.  This gives the relation
\begin{align*}
    \phi_\theta(x,t) = u_\theta(x,t) + \log Z_t. 
\end{align*}
The bounds on the ODE-SDE gap in Theorem \ref{prop:mainresult} also hold in this setting, as outlined in the remark below.

\begin{remark}
    If $\tilde R(\theta,\phi_\theta,t)<\delta$, then $W_2(p^{ODE}_\theta(\cdot,t), p^{SDE}_\theta(\cdot,t))<C\delta$ for some $C>0$ independent of $\delta$  (under the assumptions of Theorem \ref{prop:mainresult}). Indeed,
    \begin{align*}
        u_\theta(x,t) = u_\theta^{SDE}(x,t) + e_u(x,t) 
    \end{align*}
    by the definition of $e_u$,
    thus implying
    \begin{align*}
        \phi_\theta(x,t) &= u_\theta^{SDE}(x,t) +e_u(x,t) + \log Z_t \\
        &= u_\theta^{SDE}(x,t) + \log Z_T + (e_u(x,t) - \log Z_T + \log Z_t) \\
        &= u_\theta^{SDE}(x,t) + \log Z_T + \tilde{e}_u(x,t),
    \end{align*}
    where we have set 
    \begin{align*}
    \tilde{e}_u(x,t) &= e_u(x,t) - \log Z_T + \log Z_t \\
    &= \phi_\theta(x,t)- u_\theta^{SDE}(x,t)  - \log Z_T
    \end{align*}
    as the error between $\phi_\theta$ and $u^{SDE}_\theta+\log Z_T$.
Since $u_\theta^{SDE}$ is a solution of \eqref{eq:logfokkerplanckapprox}, it follows that $u_\theta^{SDE}  + \log Z_T$ is also a solution of \eqref{eq:logfokkerplanckapprox}. 
Following similar arguments as in \eqref{eq:gradbound}, we deduce that
\begin{align*}
    \int_t^T \|\nabla u_\theta^{SDE}(\cdot,s) - \nabla \phi_\theta(\cdot,s)\|_{L^2(\Omega)}^2 ds &=
    \int_t^T \|\nabla u_\theta^{SDE}(\cdot,s) + \nabla \log Z_T - \nabla \phi_\theta(\cdot,s)\|_{L^2(\Omega)}^2 ds \\
    &= \int_t^T \|\nabla \tilde{e}_u(\cdot,s)\|_{L^2(\Omega)}^2 ds \\
    &< C\delta 
\end{align*}
for some $C>0$ independent of $\delta$. Thus, $W_2(p^{ODE}_\theta(\cdot,t), p^{SDE}_\theta(\cdot,t))<C\delta$ follows by applying the steps in the proof of Theorem \ref{prop:mainresult} with $\phi_\theta$ substituted in place of $u_\theta$.

\end{remark}

\subsection{The ODE-SDE gap for the  approximate score-Fokker--Planck equation}\label{sec:result_score_FP}
In this work we   primarily focus   on the underlying mean-field behaviour observed in score-based diffusion models and thus our focus has been on the potential parameterisation discussed so far. However, in most practical implementations, a score parameterisation  is adopted due to computational efficiency, given by  $s_\theta(x,t) \approx \nabla u(x,t)= \nabla \log p(x,t)$ for the density $p$ and the log-density $u$ of the forward SDE \eqref{eq:forward}. The the score parameterisation is linked to the score-Fokker--Planck equation, see e.g.\  \cite{lai2023fpdiffusion}. To ensure applicability of our results to this case, we argue in this section that the bounds on the ODE-SDE gap in Theorem \ref{prop:mainresult} also hold for the score parameterisation. 

A Score-Fokker--Planck equation can be derived by taking the gradient of the associated log-Fokker--Planck equation. In order to derive an analogous result to Theorem \ref{prop:mainresult} we are interested in the score-Fokker--Planck equation of the approximate SDE with log-Fokker--Planck equation \eqref{eq:logfokkerplanckapprox}. Taking the gradient of \eqref{eq:logfokkerplanckapprox} and setting  $s_{\theta}^{SDE}(x,t) = \nabla u_{\theta}^{SDE}(x,t) \in \mathbb{R}^d$ yields the \emph{approximate score-Fokker--Planck equation}
\begin{align} \label{eq:score_fp}
    \begin{split}
    \frac{\partial s_{\theta}^{SDE}}{\partial t}(x,t)& + \nabla(\nabla \cdot   f_\theta^{SDE}(x,t)) + (\nabla s_{\theta}^{SDE}(x,t)) f_\theta^{SDE}(x,t) + (\nabla f_\theta^{SDE}(x,t)) s_{\theta}^{SDE}(x,t) \\
    &+ g^2(t)(\nabla s_{\theta}^{SDE}(x,t))s_{\theta}^{SDE}(x,t) + \frac{1}{2}  g^2(t) \nabla(\nabla \cdot s_{\theta}^{SDE}(x,t))=0.
    \end{split}
\end{align}
Here we use analogous notation to the potential case so that $s_{\theta}^{SDE}(x,t)$ is the true score associated with the approximate reverse SDE \eqref{eq:reverseapprox} and is linked with the density $p_\theta^{SDE}$ via $s_{\theta}^{SDE}=\nabla \log p_{\theta}^{SDE}$. Similarly, we also introduce the score $s_{\theta}^{ODE}(x,t) = \nabla u_{\theta}^{ODE}(x,t)=\nabla \log p_{\theta}^{ODE}(x,t) \in \mathbb{R}^d$ associated with the approximate probability flow ODE \eqref{eq:probflowapprox2}. 

Let $s_\theta(x,t)$ denote a score model approximating $\nabla \log p(x,t)$. Following an analogous calculation to \eqref{eq:restransform}, the residual corresponding to the approximate score-Fokker--Planck equation \eqref{eq:score_fp} can be written as the residual of a score-Fokker--Planck equation and we define the \emph{score-Fokker--Planck residual} by:
\begin{align}\label{eq:res_score}
    \begin{split} 
   R^s(\theta,s_\theta,t) &=  V(T-t)^{-1}\int_t^T\left\|
   \frac{\partial s_\theta}{\partial r}(\cdot,r) + \nabla (\nabla \cdot f(\cdot,r)) + (\nabla  s_\theta(\cdot,r)) f(\cdot,r)  \right.\\
    &\quad\quad \left. + (\nabla  f(\cdot,r))s_\theta(\cdot,r) - g^2(r)(\nabla s_\theta(\cdot,r)) s_\theta(\cdot,r)   - \frac{1}{2}g^2(r) \nabla (\nabla \cdot s_\theta(\cdot,r))\right\|_{L^2(\Omega)}^2 dr.
\end{split}
\end{align}

We can now state analogous result to Theorem~\ref{prop:mainresult} for the approximate score-Fokker--Planck equation: If $R^s(\theta,s_\theta,t)<\delta$, then $W_2(p^{ODE}_\theta(\cdot,t), p^{SDE}_\theta(\cdot,t))<C\delta$ for some $C>0$ independent of $\delta$ (under the assumptions of Theorem \ref{prop:mainresult}). This result can be derived by applying analogous Steps I and II in the proof of Theorem \ref{prop:mainresult} and generalising them to vector-valued functions as appropriate.  More precisely, Step I of the proof has to be generalised to vector-valued solutions $s_\theta^{SDE}$ of \eqref{eq:score_fp} as opposed to the scalar solution $u_\theta^{SDE}$ of \eqref{eq:logfokkerplanckapprox}, but due to the similarity of the equations this step can be done analogously. Step II follows as in the proof of Theorem \ref{prop:mainresult}. Due to the similarity of the proofs, the detailed proof is omitted here.

\section{Numerical experiments}\label{sec:numerics}
To demonstrate our analytical results numerically and ensure visual interpretability, we implement several diffusion models in $\mathbb{R}^2$ that attain a range of log-Fokker--Planck residual values \eqref{eq:lfpresid} using various datasets. For the forward SDE, we choose $f(x,t)=-x$ and $g(t)=1$, resulting in the simple Ornstein--Uhlenbeck process
\begin{align*}
    dx_t = -x_tdt + dW_t.
\end{align*}
Following \eqref{eq:logfokkerplanck}, the associated log-Fokker--Planck equation is given by
\begin{align*}
    \frac{\partial u(x,t)}{\partial t} = 2 + x \cdot \nabla u(x,t) + \frac{1}{2}\|\nabla u(x,t)\|_2^2 + \frac{1}{2}\nabla^2 u(x,t).
\end{align*}
In our experiments we take three different data distributions and train a neural network to minimise the loss function
\begin{align*}
    L(\theta,w_R) := \mathcal{L}_{DSM}(\theta,\nabla \phi_\theta, \lambda) + w_R\tilde{R}(\theta,\phi_\theta,0).
\end{align*}
for differing values of $w_R$ where $ \mathcal{L}_{DSM}$ and $\tilde{R}$ are defined in \eqref{DSM} and \eqref{eq:lfpresid}, respectively, and $\lambda$ is set according to the likelihood weighting. Note that for our specific setting, we have
\begin{align*}
    \tilde{R}(\theta,\phi_\theta,t) &= 
   \mathbb{E}_{\subalign{&s \sim U(t,T)\\ &x \sim U(\Omega)}}\left[\left(    \frac{\partial u(x,s)}{\partial t} - 2 - x \cdot \nabla u(x,s) - \frac{1}{2}\|\nabla u(x,s)\|_2^2 - \frac{1}{2}\nabla^2 u(x,s)\right)^2\right],
\end{align*}
Therefore both the denoising score matching objective $\mathcal{L}_{DSM}(\theta,\nabla \phi_\theta,\lambda)$ and the log-Fokker--Planck residual $\tilde{R}(\theta,\phi_\theta,0)$ are approximated using Monte Carlo estimation. We set $T=10$ and  the likelihood weighting implies  $\lambda(t)=1$ for $t\in [0,T]$. Note that we do not add terms that enforce boundary conditions in space or time, since the denoising score matching objective \eqref{DSM} already encourages consistency   with these conditions (up to a multiplicative constant proportional to the underlying density). We generate  samples from $p_\theta^{SDE}(\cdot,0)$ and $p_\theta^{ODE}(\cdot,0)$ using Euler-Maruyama and Euler discretisations of the reverse approximate SDE \eqref{eq:reverseapprox} and the reverse approximate probability flow ODE \eqref{eq:probflowapprox}, respectively. To validate our results we generate 3 million samples from each distribution. Due to computational constraints these samples are then discretised onto a $64\times 64$ grid. When   computing the Wasserstein distances to the target distribution and producing visualisations, we will consider  the discretised distributions. Figure \ref{fig:target_dists} shows the target distributions for our experiments. 
\begin{figure}[H]
    \centering
    \includegraphics[width=0.7\textwidth]{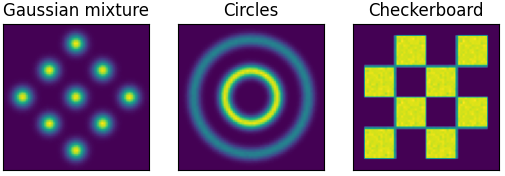}
    \caption{Our three examples include a Gaussian mixture, a non-Gaussian but smooth concentric circles distribution, and a discontinuous checkerboard distribution.}
    \label{fig:target_dists}
\end{figure} 
We parameterise our potential model $\phi_\theta(x,t)$ by a fully connected neural network with 2 hidden layers and 80 nodes per layers. We apply softplus activation functions, which have well defined first and second derivatives as required to evaluate $\tilde{R}(\theta,\phi_\theta,0)$. Each model is trained for 100,000 iterations using Adam with learning rate decaying from $10^{-3}$ down to $10^{-5}$.  Figure~\ref{fig:sample_dists} shows the samples obtained from $p_\theta ^{SDE}(\cdot,0)$ and $p_\theta ^{ODE}(\cdot,0)$ for different weighting parameters $w_R$.

\begin{figure}[htb]
    \centering
    \includegraphics[width=1\textwidth]{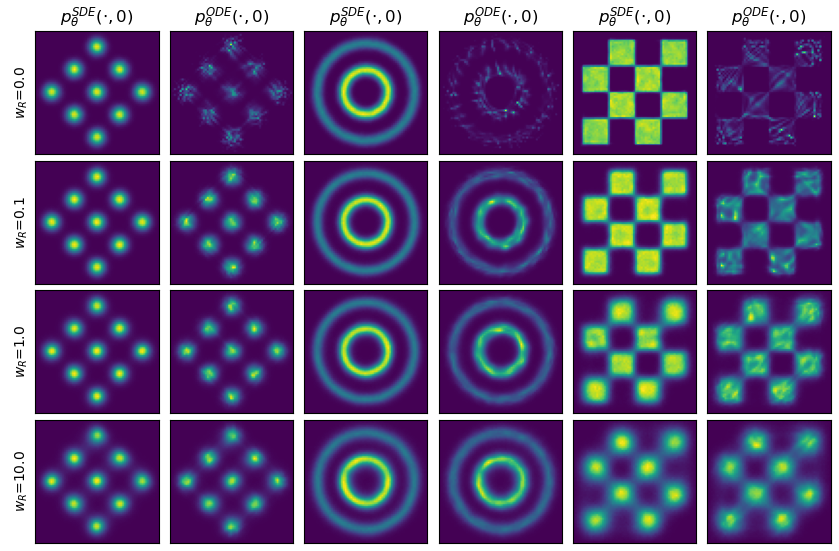}
    \caption{Distributions of $p_\theta ^{SDE}(\cdot,0)$ and $p_\theta ^{ODE}(\cdot,0)$ for weighting parameters $w_R$ taking values in $(0,0.1,1,10)$. The rows indicate which weighting parameter was used, while the columns indicate whether the displayed distribution  is of $p_\theta ^{SDE}(\cdot,0)$ or of $p_\theta ^{ODE}(\cdot,0)$ in the corresponding experiment. Samples displayed from the ODE and SDE samplers were attained using the same score model. }
    \label{fig:sample_dists}
\end{figure}

We see from Figure \ref{fig:sample_dists} that if we only optimise $\mathcal{L}_{DSM}$ (i.e.\ for $w_R=0$) the resulting $p_\theta^{ODE}(\cdot,0)$   is quite different from the true distribution. Notably, areas of high probability in $p_\theta^{ODE}(\cdot,0)$ do coincide with high probability regions of $p_\theta^{SDE}(\cdot,0)$. Therefore in typical generative modelling scenarios it may be difficult to identify this mischaracterisation of the data distribution, given that  the samples generated from $p_\theta^{ODE}(\cdot,0)$ are generally plausible. Visually, we see that adding a factor of $\tilde{R}(\theta,\phi_\theta,0)$ to the loss function initially results in an improvement in $p_\theta^{ODE}$. This is further justified in Table  \ref{tab:odetrue}, which shows that the distance between $p_\theta^{ODE}(\cdot,0)$ and $p_0$ reduces for $w_R=0.1$ and $w_R=1$ when compared to $w_R=0$. Increasing $w_R$ beyond this further reduces the gap between $p_\theta^{ODE}(\cdot,0)$ and $p_\theta^{SDE}(\cdot,0)$ as demonstrated in Table \ref{tab:odesde}, however this comes at the cost of increasing the distance from both $p_\theta^{ODE}(\cdot,0)$ and $p_\theta^{SDE}(\cdot,0)$ to $p_0$ which can  be observed in Tables \ref{tab:odetrue} and \ref{tab:sdetrue} for $w_R=10$. This can clearly be seen in Figure \ref{fig:sample_dists} by the overly smoothed distributions that are attained with higher $w_R$. Table \ref{tab:sdetrue} shows that the quality of $p_\theta^{SDE}$ degrades monotonically with increasing $w_R$ which results in the negative correlation between $\tilde R(\theta,\phi_\theta,0)$ and $\mathcal{L}_{DSM}(\theta,\nabla \phi_\theta,\lambda)$ observed in Figure \ref{fig:DSMxRes_scatter}. From this we conclude that the cost of improving $p_\theta^{ODE}$ is a reduction in the quality of $p_\theta^{SDE}$. Finally, in Figure~\ref{fig:WassxRes_scatter} we evaluate $\tilde R(\theta,\phi_\theta,0)$ for each of our trained models and visualise the relation between  $\tilde R(\theta,\phi_\theta,0)$ and the associated $W_2(p_\theta^{ODE}(\cdot,0),p_\theta^{SDE}(\cdot,0))$ values. This demonstrates a clear positive correlation supporting our theoretical analysis, where we proved an upper bound on the Wasserstein 2-distance between the ODE- and SDE-induced distributions in terms of a Fokker--Planck residual $\tilde R$.

\begin{table}[H] 
\centering
\begin{tabular}{|c|c|c|c|}
\hline 
$w_R$ & Mixture & Circles & Checkerboard \\
\hline
0.0 & 0.011599 & 0.013401 & 0.027196 \\
0.1 & 0.003107 & 0.005283 & 0.007160 \\
1   & 0.002901 & 0.003090 & 0.003528 \\
10  & 0.002208 & 0.001704 & 0.002104 \\
\hline
\end{tabular} 
\caption{Estimated values of $W^2_2(p_\theta^{ODE}(\cdot,0),p_\theta^{SDE}(\cdot,0))$ for each distribution and differing $w_R$ values.}
\label{tab:odesde}
\end{table}

\begin{table}[H] 
\centering
\begin{tabular}{|c|c|c|c|}
\hline 
$w_R$ & Mixture & Circles & Checkerboard \\
\hline
0.0 & 0.010967 & 0.014354 & 0.024416 \\
0.1 & 0.006039 & 0.007373 & 0.011386 \\
1   & 0.005280 & 0.008280 & 0.011428 \\
10  & 0.017225 & 0.009378 & 0.033052 \\
\hline
\end{tabular}
\caption{Estimated values of $W^2_2(p_\theta^{ODE}(\cdot,0),p_0)$ for each distribution and differing $w_R$ values.}
\label{tab:odetrue}
\end{table}

\begin{table}[H] 
\centering
\begin{tabular}{|c|c|c|c|}
\hline 
$w_R$ & Mixture & Circles & Checkerboard \\
\hline
0.0 & 0.002180 & 0.002268 & 0.003681 \\
0.1 & 0.002940 & 0.003323 & 0.005387 \\
1   & 0.004271 & 0.004541 & 0.011552 \\
10  & 0.015989 & 0.009097 & 0.032158 \\
\hline
\end{tabular}
\caption{Estimated values of $W^2_2(p_\theta^{SDE}(\cdot,0),p_0)$ for each distribution and differing $w_R$ values.}
\label{tab:sdetrue}
\end{table}

\begin{figure}[H]
    \centering
    \includegraphics[width=0.7\textwidth]{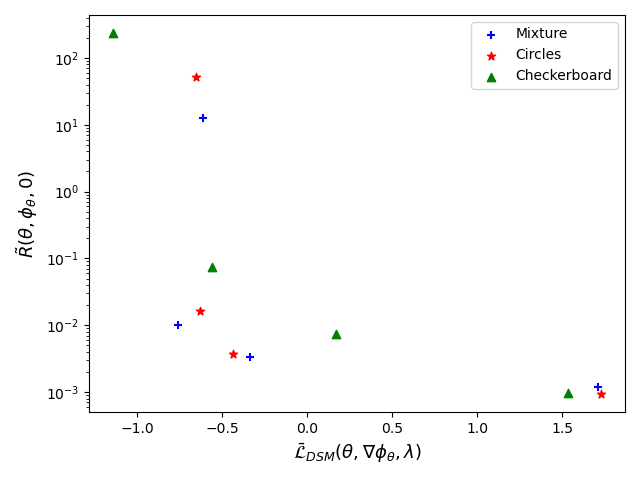}
    \caption{Comparison of the Fokker--Planck residual values against the denoising score matching loss. We see that the attainment of a low $\tilde{R}$ is correlated with a higher $\mathcal{L}_{DSM}$, thus explaining the degradation in sample quality for high $w_R$. Here the $\mathcal{L}_{DSM}$ have been normalised by data distribution to eliminate bias specific to each dataset.}
    \label{fig:DSMxRes_scatter}
\end{figure}

\begin{figure}[H]
    \centering
    \includegraphics[width=0.7\textwidth]{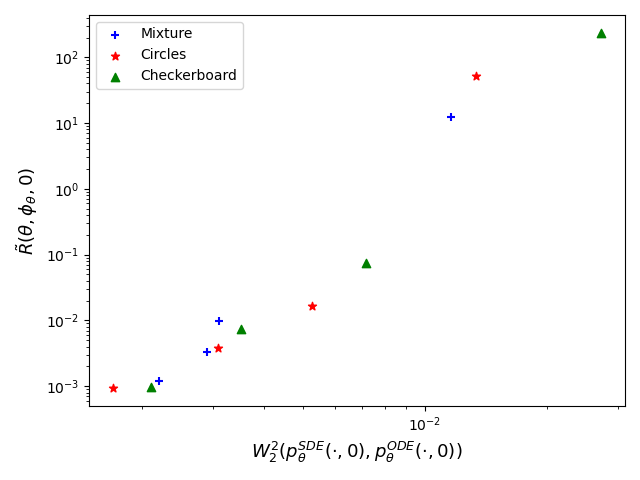}
    \caption{Comparison of the Fokker--Planck residual values against the Wasserstein distance between the ODE and SDE samples.}
    \label{fig:WassxRes_scatter}
\end{figure}

\section{Conclusions}
In this work, we conducted a systematic investigation into the dynamics that arise in score-based diffusion models. We mainly focused on the differences between the generative densities $p_\theta^{SDE}$ and $p_\theta^{ODE}$ defined by the reverse approximate SDE and the approximate probability flow ODE, respectively.  Analytically, we proved that the discrepancy between $p_\theta^{SDE}$ and $p_\theta^{ODE}$  can be bounded by a Fokker--Planck residual in the Wasserstein 2-distance, thus giving a deeper insight into the connection between the two generative distributions in terms of the Fokker--Planck dynamics underlying the diffusion process. Numerically, we showed that  $p_\theta^{SDE}$ and $p_\theta^{ODE}$ can differ substantially when the neural network is trained using the standard score-matching objective. Our numerical experiments also demonstrate that penalising the loss function by the Fokker--Planck residual indeed leads to closing the gap between the ODE and the SDE distributions  in the Wasserstein 2-distance. Our findings revealed that imposing this additional constraint within our loss function could improve the quality of $p_\theta^{ODE}$ when compared to the ground truth, though in exchange for this we observed concurrent degradation in the quality of $p_\theta^{SDE}$. The practical implication of these findings is that enforcing self-consistency through penalisation by the Fokker--Planck residual is unlikely to improve state-of-the-art generation using stochastic samplers. However, for downstream tasks where deterministic generation is required (e.g.\ optimisation-based approaches to inverse problems), such penalisation could provide a potential avenue to improve sample quality.

\section*{Acknowledgements}
TD, LMK, CBS and CB acknowledge support from the EPSRC programme grant in `The Mathematics of Deep Learning', under the project EP/L015684/1 which also funded TD’s 6-month research visit at the University of Cambridge.
JS, LMK and CBS acknowledge support from the Cantab Capital Institute for the Mathematics of Information. 
JS additionally acknowledges the support from Aviva. 
CBS additionally acknowledges support from the Philip Leverhulme Prize, the Royal Society Wolfson Fellowship, the EPSRC grants EP/S026045/1 and EP/T003553/1, EP/N014588/1, EP/T017961/1, the Wellcome Innovator Award RG98755 and the Alan Turing Institute.

\appendix
\section{Lower bound on $\nabla^2u_\theta^{SDE}$}

The proof of Theorem \ref{prop:mainresult} requires a preliminary result on the existence of a lower bound of $\nabla^2u_\theta^{SDE}$, given by the following lemma.

\begin{lemma}\label{lem:diffusionbound}
Let $t\in[0,T]$, let   $\Omega \subset \mathbb{R}^d$ be a bounded domain   with $\partial \Omega \in C^\infty$, and assume that Assumptions \ref{ass:regularity} hold. Assume that $p^{SDE}_\theta$ satisfies \eqref{eq:fokkerplanckapprox} on   $\Omega$ with terminal condition  $\pi$ restricted to $\Omega$ and Dirichlet boundary conditions $p_B^{SDE}$ on $\partial \Omega \times [0,T]$, with $u_\theta^{SDE}= \log p_\theta^{SDE}$. 
Then there exists $C \in \mathbb{R}$ such that
    $\nabla^2 \bar u^{SDE}_\theta \geq C$ on $\Omega \times [0,T]$ for $\bar u^{SDE}_\theta(\cdot,\tau)=u^{SDE}_\theta(\cdot,T-\tau) $ for $\tau\in [0,T]$.
\end{lemma}
\begin{proof}
    By construction $u_{\theta}^{SDE} = \log p_{\theta}^{SDE}$ where $p_{\theta}^{SDE}$ solves \eqref{eq:fokkerplanckapprox} with terminal condition $p_{\theta}^{SDE}(x,T)=\pi$ and Dirichlet boundary data $p_B^{SDE}$. Using the reverse time variable $\tau=T-t$, we can write the reverse dynamics for $\bar p^{SDE}_\theta(x,\tau)$ as
    \begin{align*}
        \frac{\partial \bar p_{\theta}^{SDE}}{\partial \tau}(x,\tau) = \nabla \cdot (\bar f^{SDE}_{\theta}(x,\tau)\bar p_{\theta}^{SDE}(x,\tau)) +\frac{1}{2}\bar g^2(\tau)\nabla^2(\bar p_{\theta}^{SDE}(x,\tau))
    \end{align*}
    with initial data $\bar p^{SDE}_\theta (\cdot,0)=\pi$ and  Dirichlet boundary conditions $\bar p^{SDE}_\theta(x,\tau)=p^{SDE}_B(x,T-\tau)$ for $(x,\tau)\in \partial \Omega \times [0,T]$.
    Under Assumptions \ref{ass:regularity} on $f,g,u_\theta$, we have that $\bar p_{\theta}^{SDE} \in C^\infty(\Omega \times [0,T])$ by \cite[Thm. 8.3.4]{wu2006elliptic}. Furthermore $\bar p_{\theta}^{SDE}$ is strictly positive on $\Omega$.
    Since $\bar p_{\theta}^{SDE} \in C^\infty(\Omega \times [0,T])$ there exist $k_1,k_2>0$ such that $\|\nabla \bar p_{\theta}^{SDE}(x,\tau)\|_2<k_1$ and $|\nabla^2 \bar p_\theta^{SDE}(x,\tau)|<k_2$ for all $(x,\tau)\in \Omega\times [0,T]$. As $\bar p_{\theta}^{SDE}$ is strictly positive, there exists $k_0>0$ such that  $\bar p_{\theta}^{SDE}(x,\tau)>k_0$ on $\Omega\times [0,T]$. Computing directly from $\bar u_\theta^{SDE}=\log \bar p_\theta^{SDE}$ we obtain that
    \begin{align*} 
            \nabla^2 \bar u_{\theta}^{SDE}(x,\tau) &= \frac{\bar p_{\theta}^{SDE}(x,\tau)\nabla^2 \bar p_{\theta}^{SDE}(x,\tau) - \|\nabla \bar p_{\theta}^{SDE}(x,\tau)\|_2^2 }{\bar p_{\theta}^{SDE}(x,\tau)^2} \\
            &= \frac{\nabla^2 \bar p_{\theta}^{SDE}(x,\tau)}{\bar p_{\theta}^{SDE}(x,\tau)} - \left\|\frac{\nabla \bar p_{\theta}^{SDE}(x,\tau)}{\bar p_{\theta}^{SDE}(x,\tau)}\right\|_2^2.
    \end{align*}
    This results in the  bound
    \begin{align*}
            |\nabla^2 \bar u_{\theta}^{SDE}(x,\tau)|  \leq  \left|\frac{\nabla^2 \bar p_{\theta}^{SDE}(x,\tau)}{\bar p_{\theta}^{SDE}(x,\tau)}\right| + \left\|\frac{\nabla \bar p_{\theta}^{SDE}(x,\tau)}{\bar p_{\theta}^{SDE}(x,\tau)}\right\|_2^2  \leq \left|\frac{k_2}{k_0}\right| + \left(\frac{k_1}{k_0}\right)^2
    \end{align*}    
    which yields the required bound.
\end{proof}

\printbibliography

@inproceedings{song2021sde,
title={Score-Based Generative Modeling through Stochastic Differential Equations},
author={Yang Song and Jascha Sohl-Dickstein and Diederik P Kingma and Abhishek Kumar and Stefano Ermon and Ben Poole},
booktitle={International Conference on Learning Representations},
year={2021}
}

@inproceedings{dhariwal2021diffusion_beats_gans,
title={Diffusion Models Beat {GAN}s on Image Synthesis},
author={Prafulla Dhariwal and Alexander Quinn Nichol},
booktitle={Advances in Neural Information Processing Systems},
editor={A. Beygelzimer and Y. Dauphin and P. Liang and J. Wortman Vaughan},
year={2021}
}

@article{song2021maximum,
  title={Maximum likelihood training of score-based diffusion models},
  author={Song, Yang and Durkan, Conor and Murray, Iain and Ermon, Stefano},
  journal={Advances in Neural Information Processing Systems},
  volume={34},
  pages={1415--1428},
  year={2021}
}

@InProceedings{sohldickstein2015diffusion_original,
  title = 	 {Deep Unsupervised Learning using Nonequilibrium Thermodynamics},
  author = 	 {Sohl-Dickstein, Jascha and Weiss, Eric and Maheswaranathan, Niru and Ganguli, Surya},
  booktitle = 	 {Proceedings of the 32nd International Conference on Machine Learning},
  pages = 	 {2256--2265},
  year = 	 {2015},
  editor = 	 {Bach, Francis and Blei, David},
  volume = 	 {37},
  series = 	 {Proceedings of Machine Learning Research},
  address = 	 {Lille, France},
  month = 	 {07--09 Jul},
  publisher =    {PMLR}
}

@article{score_matching,
  author  = {Aapo Hyv{{\"a}}rinen},
  title   = {Estimation of Non-Normalized Statistical Models by Score Matching},
  journal = {Journal of Machine Learning Research},
  year    = {2005},
  volume  = {6},
  number  = {24},
  pages   = {695--709}
}

@ARTICLE{vincent2011connection,
  author={Vincent, Pascal},
  journal={Neural Computation}, 
  title={A Connection Between Score Matching and Denoising Autoencoders}, 
  year={2011},
  volume={23},
  number={7},
  pages={1661-1674}
}

@article{anderson1982reverse_time_sde,
title = {Reverse-time diffusion equation models},
journal = {Stochastic Processes and their Applications},
volume = {12},
number = {3},
pages = {313-326},
year = {1982},
issn = {0304-4149},
author = {Brian D.O. Anderson}
}

@article{ho2020denoising,
  title={Denoising diffusion probabilistic models},
  author={Ho, Jonathan and Jain, Ajay and Abbeel, Pieter},
  journal={Advances in Neural Information Processing Systems},
  volume={33},
  pages={6840--6851},
  year={2020}
}

@inproceedings{kingmaVDM,
 author = {Kingma, Diederik and Salimans, Tim and Poole, Ben and Ho, Jonathan},
 booktitle = {Advances in Neural Information Processing Systems},
 editor = {M. Ranzato and A. Beygelzimer and Y. Dauphin and P.S. Liang and J. Wortman Vaughan},
 pages = {21696--21707},
 publisher = {Curran Associates, Inc.},
 title = {Variational Diffusion Models},
 volume = {34},
 year = {2021}
}

@book{wu2006elliptic,
  title={Elliptic And Parabolic Equations},
  author={Wu, Z. and Yin, J. and Wang, C.},
  isbn={9789813101708},
  year={2006},
  publisher={World Scientific Publishing Company}
}

@inbook{song2020generative_score,
author = {Song, Yang and Ermon, Stefano},
title = {Generative Modeling by Estimating Gradients of the Data Distribution},
year = {2019},
publisher = {Curran Associates Inc.},
address = {Red Hook, NY, USA},
booktitle = {Proceedings of the 33rd International Conference on Neural Information Processing Systems},
articleno = {1067},
numpages = {13}
}

@inproceedings{rombach2022stable_diffusion,
  title={High-resolution image synthesis with latent diffusion models},
  author={Rombach, Robin and Blattmann, Andreas and Lorenz, Dominik and Esser, Patrick and Ommer, Bj{\"o}rn},
  booktitle={Proceedings of the IEEE/CVF conference on computer vision and pattern recognition},
  pages={10684--10695},
  year={2022}
}

@misc{ramesh2022dalle,
      title={Hierarchical Text-Conditional Image Generation with CLIP Latents}, 
      author={Aditya Ramesh and Prafulla Dhariwal and Alex Nichol and Casey Chu and Mark Chen},
      year={2022},
      eprint={2204.06125},
      archivePrefix={arXiv},
      primaryClass={cs.CV}
}

@inproceedings{lu2022maximumODE,
  title={Maximum Likelihood Training for Score-Based Diffusion ODEs by High-Order Denoising Score Matching}, 
  author={Lu, Cheng and Zheng, Kaiwen and Bao, Fan and Chen, Jianfei and Li, Chongxuan and Zhu, Jun},
  booktitle={International Conference on Machine Learning},
  pages={14429--14460},
  year={2022},
  organization={PMLR}
}

@inproceedings{lai2023fpdiffusion,
author = {Lai, Chieh-Hsin and Takida, Yuhta and Murata, Naoki and Uesaka, Toshimitsu and Mitsufuji, Yuki and Ermon, Stefano},
title = {FP-Diffusion: Improving Score-Based Diffusion Models by Enforcing the Underlying Score Fokker-Planck Equation},
year = {2023},
publisher = {JMLR.org},
booktitle = {Proceedings of the 40th International Conference on Machine Learning},
articleno = {758},
numpages = {34},
location = {Honolulu, Hawaii, USA},
series = {ICML'23}
}

@inproceedings{arjovsky2017principled,
  author       = {Mart{\'{\i}}n Arjovsky and
                  L{\'{e}}on Bottou},
  title        = {Towards Principled Methods for Training Generative Adversarial Networks},
  booktitle    = {5th International Conference on Learning Representations, {ICLR} 2017,
                  Toulon, France, April 24-26, 2017, Conference Track Proceedings},
  year         = {2017}
}

@inproceedings{arjovsky2017wasserstein,
  title={Wasserstein Generative Adversarial Networks},
  author={Arjovsky, Martin and Chintala, Soumith and Bottou, L{\'e}on},
  booktitle={International conference on machine learning},
  pages={214--223},
  year={2017},
  organization={PMLR}
}

@inproceedings{salimans2021potential,
title={Should {EBM}s model the energy or the score?},
author={Tim Salimans and Jonathan Ho},
booktitle={Energy Based Models Workshop - ICLR 2021},
year={2021}
}

\end{document}